%% file: FairSmooth.tex
\documentclass[nohyperref]{article}

\usepackage{microtype}
\usepackage{graphicx}
\usepackage{subfigure}
\usepackage{booktabs} 

\usepackage{hyperref}

\usepackage[accepted]{icml2022}

\usepackage{amsmath}
\usepackage{amssymb}
\usepackage{mathtools}
\usepackage{amsthm}

\usepackage[capitalize,noabbrev]{cleveref}

\usepackage[textsize=tiny]{todonotes}

\icmltitlerunning{Input-agnostic Certified Group Fairness via Gaussian Parameter Smoothing}

\usepackage{amssymb,amsmath,amsfonts,epsfig,graphics,enumerate,float,subfigure}
\usepackage{cases}
\usepackage{color}
\usepackage{upgreek}
\usepackage{xspace}
\usepackage{float}
\usepackage{latexsym}
\usepackage{url}
\usepackage{algorithm}
\usepackage{algorithmic}
\usepackage{graphicx}
\usepackage{subfigure}
\usepackage{bigdelim}
\usepackage{multicol}
\usepackage{multirow}
\usepackage{rotating}
\usepackage{caption}
\usepackage{extarrows}

\newtheorem{definition}{Definition}
\newtheorem{theorem}{Theorem}

\newtheorem{lemma}{Lemma}

\makeatletter
\newcommand\notsotiny{\@setfontsize\notsotiny{8.7}{10}}
\makeatother

\newcommand{\method}{FairSmooth}
\newcommand{\methodI}{FairSmooth-S}
\newcommand{\methodII}{FairSmooth-D}

\begin{document}

\twocolumn[
\icmltitle{Input-agnostic Certified Group Fairness via Gaussian Parameter Smoothing}



\icmlsetsymbol{equal}{*}

\begin{icmlauthorlist}
\icmlauthor{Jiayin Jin}{AU}
\icmlauthor{Zeru Zhang}{AU}
\icmlauthor{Yang Zhou}{AU}
\icmlauthor{Lingfei Wu}{JD}
\end{icmlauthorlist}

\icmlaffiliation{AU}{Auburn University, USA}
\icmlaffiliation{JD}{JD.COM Silicon Valley Research Center, USA}

\icmlcorrespondingauthor{Yang Zhou}{yangzhou@auburn.edu}
\icmlcorrespondingauthor{Lingfei Wu}{lwu@email.wm.edu}

\icmlkeywords{Machine Learning, ICML}

\vskip 0.3in
]



\printAffiliationsAndNotice{}  

\begin{abstract}
Only recently, researchers attempt to provide classification algorithms with provable group fairness guarantees.
Most of these algorithms suffer from harassment caused by the requirement that the training and deployment data follow the same distribution.
This paper proposes an input-agnostic certified group fairness algorithm, {\sc \method}, for improving the fairness of classification models while maintaining the remarkable prediction accuracy. A Gaussian parameter smoothing method is developed to transform base classifiers into their smooth versions. 
An optimal individual smooth classifier is learnt for each group with only the data regarding the group and an overall smooth classifier for all groups is generated by averaging the parameters of all the individual smooth ones. By leveraging the theory of nonlinear functional analysis, the smooth classifiers are reformulated as output functions of a Nemytskii operator. Theoretical analysis is conducted to derive that the Nemytskii operator is smooth and induces a Fr\'echet differentiable smooth manifold. 
We theoretically demonstrate that the smooth manifold has a global Lipschitz constant that is independent of the domain of the input data, which derives the input-agnostic certified group fairness.
\end{abstract}

\vspace{-0.15cm}
\section{Introduction}\label{sec.introduction}
\vspace{-0.1cm}
\input{Introduction}


\vspace{-0.1cm}
\section{Preliminaries}\label{sec.problem}
\vspace{-0.1cm}
\input{Problem}

\vspace{-0.1cm}
\section{Certifying the Group Fairness}\label{sec:KDE}
\vspace{-0.1cm}
\input{Approach}

\vspace{-0.2cm}
\section{Experiments}\label{sec.experiment}
\vspace{-0.25cm}
\input{Experiment}

\vspace{-0.15cm}
\section{Related Work}\label{sec.related}
\vspace{-0.1cm}
\input{Related}

\vspace{-0.1cm}
\section{Conclusions}\label{sec.conclusions}
\vspace{-0.15cm}

In this work, we proposed an input-agnostic certified group fairness algorithm for improving the fairness of classification models while maintaining the remarkable prediction accuracy. First, a Gaussian parameter smoothing method is proposed to transform base classifiers into their smooth versions by averaging its prediction over Gaussian perturbations of the former's parameter within its neighborhood. 
Second, an overall smooth classifier for all groups is generated by averaging the parameters of all the optimal individual smooth classifiers.
Finally, the theoretical analysis is conducted to verify that the smooth classifier is able to achieve the certified fairness guarantees for any input data.
\vspace{-0.2cm}

%
%


\bibliographystyle{icml2022}

\input{FairSmooth.bbl}
\newpage
\appendix
\onecolumn

\section{Appendix}\label{sec.appendix}
\input{Appendix}

\end{document}

%% file: Introduction.tex
Machine learning algorithms that are deployed in human-facing applications (e.g., loan granting~\cite{AJC88}, criminal recidivism~\cite{Nort12}, job screening~\cite{NYT15}, and predictive policing~\cite{AAAS16}) are often faced with unfair prediction results by discriminating against particular groups. 
Developing fair classification algorithms with respect to sensitive attributes has become an important problem due to the growing importance of addressing social biases in machine learning~\cite{CHKV19}.
A fair classification model aims to build a classifier that not only makes accurate predictions but also prevents itself from acting against specific groups in a discriminatory way, i.e., the classifier should have similar performance for all groups~\cite{IKL21}.

To minimize discrimination effects caused by the classification models, a large number of fair classification techniques have been proposed to utilize various notions of group fairness to develop tailored classification algorithms that predict in the same way across different demographic groups~\cite{ZVGG17a,ZVGG17b,GYF18,KXPK18,MeWi18,MDJW20,WGNC20,YCK20,ZTLL20,LBCL20,PST21,CHKV21,GDL21,ChMr21,RLWS21,SGJ21,QPLH21,LiWa21,MMYS21,CMV21,DMWT21,RLWS21,Anon22a,ZLL22}.
Traditional group fairness techniques often follow manually-crafted heuristics to generate fair classifiers by solving non-convex optimization problems~\cite{ZVGG17a,ZVGG17b,KXPK18}. Such heuristic strategies plus non-convex optimization are hard to provide provable guarantees for group fairness certification.

Only recently, few studies have sought to 
provide provable group fairness guarantees for 
classification algorithms~\cite{FFMS15,MOW17,CHKV19,ULP19,GFDS21,KJWC21,SAPB21,CHKV21,IKL21,Anon22b}. Despite achieving remarkable success, 
most of the above fairness guarantees rely on the assumption that the data the model being trained with and the data encountered after deployment follow the same distribution. However, this assumption is false for many real-world problems~\cite{ZQDX21}. It is demonstrated that these models often violate the fairness guarantees and exhibit unfair bias when evaluated on data from a different distribution~\cite{Anon22b}. 
Consider an example of job screening, applicants' demographics can shift over time, which results in the change of the distribution of applicants between training and deployment datasets. As a result, the derived fairness guarantees based on the training data will not hold on the deployment data in practice, if the learning algorithms assume that the training and deployment distributions are the same.
Shifty is the first classification strategy that provides high-confidence guarantees based on user-specified fairness constraints under demographic shift between training and deployment~\cite{Anon22b}. However, Shifty needs to know the old and new demographic proportions regarding the demographic shift at training, or the demographic proportions are bounded in known intervals in Shifty, which limits the applicability of Shifty in real-world scenarios.

To our best knowledge, this work is the first to certify the group fairness of classifiers with theoretical input-agnostic guarantees, while there is no need to know the shift between training and deployment datasets with respect to sensitive attributes, by leveraging the theory of nonlinear functional analysis, including Nemytskii operator and Fr\'echet differentiable smooth manifolds.

Randomized smoothing has achieved the state-of-the-art certified robustness guarantees against worst-case attacks by smoothing with isotropic Gaussian distribution~\cite{LAGH19,CRK19,LYCJ19,LCWC19,SLRZ19,LeFe20,ZDHZ20,FWCZ20,DHBK20,Haye20}.  
One advantage of the randomized smoothing techniques is that they are agnostic to input data as well as network architectures~\cite{CRK19,MKWC20,BDMZ20,CLYL21}.
BeFair presents a new notion of group fairness: best-effort fairness, in which the performance of an overall fair classifier $f(x)$ for all groups on each group $G_k$ $(k = 1,2,\cdots,K)$ should be close to the optimal individual classifier $f_k(x)$ for each $G_k$ ($f_k$ is trained with only the data $x \in \mathbb{R}^{n}$ regarding $G_k$)~\cite{KJWC21}.

This motivates us to propose an $\epsilon$-fairness metric that measures the gap between the predictions achieved by the overall and individual classifiers and to develop a Gaussian parameter smoothing framework for the certification of the input-agnostic group fairness. 

First, a Gaussian parameter smoothing method is proposed to transform the base individual and overall classifiers $f(x;W_k)$ and $f(x;W)$ into their smooth versions below.

\vspace{-0.35cm}
{\small
\begin{equation} \label{eq:Nemytskii1}
\begin{split}
&\hat{f}(x;W_k) = \underset{\Delta}{\mathbb{E}}\big(f(x;W_k+\Delta)\big),\;\text{and}\\
&\hat{f}(x;W) = \underset{\Delta }{\mathbb{E}}\big(f(x;W+\Delta)\big),\ \Delta \sim \mathcal{N}(0,\sigma^2 I)
\end{split}
\end{equation}
}
\vspace{-0.3cm}

where $\mathcal{N}(0,\sigma^2 I)$ is the isotropic Gaussian distribution. 
$W_k$ and $W$ are the parameters of $f(x;W_k)$ and $f(x;W)$. 

Second, we train an optimal individual smooth classifier $\hat f(x;W^\ast_k)$ for each group $G_k$ with only the data regarding $G_k$, where $W^\ast_k$ is the parameter of $\hat f(x;W^\ast_k)$. An overall classifier $\hat f(x;W^\ast)$ for all groups is generated by averaging the parameters of all the individual classifiers, i.e., $W^\ast = \frac{W^\ast_1+\cdots+W^\ast_K}{K}$. Ideally, when the parameters $W^\ast_k$ of $\hat f(x;W^\ast_k)$ are close to each other, then the overall classifier $\hat f(x;W^\ast)$ with the parameter $W^\ast$ is relatively fair to all groups. 

Third, by leveraging the theory of nonlinear functional analysis, the smooth classifiers $\hat{f}(x;W)$ and $\hat{f}(x;W_k)$ are reformulated as as output functions of a Nemytskii operator $\hat{N}(W)(\cdot)$, which maps the parameter space to a function space. 
The theoretical analysis is conducted to derive that the Nemytskii operator $\hat{N}(W)(\cdot)$ is smooth and induces a Fr\'echet differentiable smooth manifold $\hat N(\mathbb{R}^m)$. 
We further prove that the smooth manifold has a global Lipschitz constant $\frac{1}{\sqrt{2\pi}\sigma}$, i.e. $\|\hat{N}(W_1)(x) - \hat{N}(W_2)(x)\|\le \frac{\|W_1-W_2\|_2}{\sqrt{2\pi}\sigma}$ for any $W_1,W_2\in\mathbb{R}^m$.  Moreover, the global Lipschitz constant $\frac{1}{\sqrt{2\pi}\sigma}$ is independent of the structure of classifiers (e.g., neural networks) and the domain of the input data. This global Lipschitz property of $\hat{N}(W)$ allows to measure the difference between two smoothed classifiers $\hat f(x;W_1)$ and $\hat f(x;W_2)$ with the distance $\|W_1 - W_2\|_2$ between their parameters. This is the reason why we perform the Gaussian parameter smoothing on classifiers. On the other hand, the base classifiers $f(x;W_k)$ and $f(x;W)$ are only locally Lipschitz regarding $W_k$ and $W$ respectively. Thus, the Lipschitz constants depend on the input data $x$.

Last but not least, with the global Lipschitz property of $\hat{N}(W)$, our training aims to minimize the training loss of each individual smooth classifier $\hat{N}(W_k)(x)$ for group $G_k$ as well as the distance among the parameters of individual smooth classifiers, where a smaller distance indicates a better group fairness. We demonstrate the input-agnostic certified group fairness, i.e., the overall smooth classifier $\hat{N}(W^\ast)(x)$ has certified $\frac{(K-1)d}{\sqrt{2\pi}K\sigma}$-fairness, where $d = \underset{1\le i,j\le K}{\max}\|W^\ast_i - W^\ast_j\|$. Namely, for $k = 1,\cdots ,K$, it holds  $|\hat{N}(W^\ast)(x) - \hat{N}(W_k^\ast)(x)|\le \frac{(K-1)d}{\sqrt{2\pi}K\sigma}$ for any input $x\in G_k$. The certified $\frac{(K-1)d}{\sqrt{2\pi}K\sigma}$-fairness of $\hat{N}(W^\ast)(x) $ guarantees that the overall smooth classifier $\hat{N}(W^\ast)(x)$ is close to each optimal individual smooth one $\hat{N}(W_k^\ast)(x)$ within distance $\frac{(K-1)d}{\sqrt{2\pi}K\sigma}$. 

In comparison with existing group fairness techniques with provable guarantees, our Gaussian parameter smoothing method based on the theory of nonlinear functional analysis exhibits two compelling advantages: (1) It is agnostic to the input data without the knowledge or assumption about the distribution of training and deployment data; (2) Producing the overall fair classifier from the optimal individual classifier for each group is able to recover the most accurate model for each group in an ideal setting, which is often ignored by other fair classification algorithms.

Empirical evaluation on real datasets demonstrates the superior performance of our Gaussian parameter smoothing model against several state-of-the-art group fairness techniques with provable guarantees. In addition, more experiments, implementation details, and hyperparameter selection and setting are presented in Appendices~\ref{sec.AdditionalExperiments}-\ref{sec.ExperimentDetails}.

%% file: Problem.tex
\vspace{-0.1cm}
\subsection{Problem Statement}\label{sec.ProblemStatement}
\vspace{-0.1cm}

Given a dataset $D = X \times Z \times Y$ where $X$, $Z$, and $Y$ represent the feature, group attribute, and label spaces respectively. Each sample $(x, z, y)$ is drawn from $D$ where $x$ is an $n$-dimensional feature vector $x = [x_1, \ldots, x_n] \in \mathbb{R}^n$, $z$ is one of $K$ sensitive attribute values that are relevant for fairness (e.g., gender, age, or race), i.e., $z \in Z = \{z_1, \cdots, z_K\}$, and $y$ is a class label $y \in Y = \{0, 1\}$.
Without loss of generality, we discuss the case where there is only one sensitive attribute $Z$ in the paper. Our input-agnostic certified group fairness algorithm can be easily extended to the case of multiple sensitive attributes.
Based on $K$ sensitive attribute values, the dataset $D$ is divided into $K$ non-overlapping groups $G_k=\{(x, z, y) \in D | z = z_k\}$  and $G_k \cap G_j = \O$ for all $1 \leq j\neq k \leq K$.

Let $f(x;W_k)$ denote a deterministic classifier for group $G_k$ trained with only the data regarding $G_k$. Our Gaussian parameter smoothing method converts each individual base classifier $f(x;W_k)$ into an individual smooth one $\hat{f}(x;W_k)$.

\vspace{-0.35cm}
{\small
\begin{equation} \label{eq:Gaussian}
\hat{f}(x;W_k) = \underset{\Delta}{\mathbb{E}}\big(f(x;W_k+\Delta)\big),\ \Delta \sim \mathcal{N}(0,\sigma^2 I)
\end{equation}
}
\vspace{-0.5cm}

where $\mathcal{N}(0,\sigma^2 I)$ is the isotropic Gaussian distribution. The training objective of classification is to make the parameters of all individual smooth classifiers close to each other, while achieving the optimal prediction accuracy by each $\hat f(x;W_k)$.

\vspace{-0.35cm}
{\small
\begin{equation} \label{eq:Loss}
\begin{split}
\min_{W_1,\cdots,W_K} &\sum_{k=1}^{K} \underset{(x_i, z_i, y_i) \in G_k} {\mathbb{E}}\Big(\mathcal{L}\big(\hat{f}(x_i;W_k), y_i\big)\Big) + \\
\alpha &\sum_{k=1}^{K} \sum_{l>k}^{K} \left\|W_k-W_l\right\|^2_{2}
\end{split}
\end{equation}
}
\vspace{-0.5cm}

\noindent where $\mathcal{L}$ denotes the loss function of classification, e.g., cross-entropy. The second term is a parameter disparity term that forces all individual classifiers to approach each other, in order to alleviate the unfairness issue.

An $\epsilon$-fairness metric will be introduced to measure the gap between the predictions achieved by the overall smooth classifier $\hat{f}(x;W^\ast)$ and the individual smooth ones $\hat{f}(x;W^\ast_k)$, where the parameter $W^\ast$ of $\hat{f}(x;W^\ast)$ is generated by averaging the parameters of all $\hat{f}(x;W^\ast_k)$, i.e., $W^\ast = \frac{W^\ast_1+\cdots+W^\ast_K}{K}$.

In an ideal scenario where enough training data are given for each group, the training of individual smooth classifiers  would produce the accurate prediction for each group as well as achieve the remarkable fairness among different groups, which will be certified with the fairness guarantees.

\vspace{-0.1cm}
\subsection{Theory of Nonlinear Functional Analysis}\label{sec.Theory}
\vspace{-0.1cm}

In mathematics, nonlinear functional analysis aims to study nonlinear mappings (i.e., nonlinear operators) between infinite-dimensional vector spaces and certain classes of nonlinear spaces and their mappings~\cite{RiSz55}. 

A Cauchy sequence is a sequence whose elements become arbitrarily close to each other as the sequence progresses~\cite{Lang93}.
\begin{definition}[Cauchy Sequence]
Let $(X,\|\cdot \|)$ be a metric space. A sequence $\{x_i\}_{i=1}^\infty \subset X$ is called a Cauchy sequence if for any $\varepsilon > 0$, there exists an integer $I(\varepsilon)>0$, such that for any integers $i,j > I(\varepsilon)$, it holds $\|x_i - x_j\|<\varepsilon$.
\end{definition}

A Banach space is a vector space with a metric that allows the computation of vector magnitude and distance between vectors and is complete in the sense that a Cauchy sequence of vectors always converges to a well-defined limit that is within the space~\cite{RiSz55}.

\begin{definition}[Banach Space]
A Banach space is a complete metric space, where a metric space $(X,\|\cdot \|)$ is complete if any Cauchy sequence in $X$ has a limit. That is, for any Cauchy sequence $\{x_i\}_{i=1}^\infty \subset X$, there exists $x\in X$, such that $\lim\limits_{n\rightarrow\infty}\|x_i - x\| = 0$. 
\end{definition}

\begin{definition}[$L^p$ Space]
$(\mathbb{R}^n, \|\cdot \|_{L^p})$ means the vector space $\mathbb{R}^n$ equipped with the $L^p$-norm, i.e.,

\vspace{-0.35cm}
{\small
\begin{equation}
\begin{split}
&\|(x_1,x_2,\cdots,x_n)\|_{L^p} = (\sum_{k=1}^n |x_n|^p)^{1/p} \text{for}\;1\le p<\infty; \;\\
&\|(x_1,x_2,\cdots,x_n)\|_{L^\infty} = \underset{1\le k\le n}{\sup}|x_k|
\end{split}
\end{equation}
}
\vspace{-0.5cm}
\end{definition}

The $L^p$ function spaces are defined using a natural generalization of the $L^p$-norm for finite-dimensional vector spaces.

\begin{definition}[$L^p(\Omega)$ Space]
For $1\le p<\infty$, 

\vspace{-0.35cm}
{\small
\begin{equation}
L^p(\Omega)= \{f(x): \Omega\subset\mathbb{R}^n\rightarrow \mathbb{R}, \int_{\Omega} |f(x)|^p dx<\infty\}
\end{equation}
}
\vspace{-0.5cm}

equipped with the norm $\|f(x)\|_{L^p} = (\int_{\Omega} |f(x)|^p dx)^{1/p}$; 

For $p =\infty$,

\vspace{-0.35cm}
{\small
\begin{equation}
L^\infty(\Omega)=\{f(x): \Omega\subset\mathbb{R}^n\rightarrow \mathbb{R}, \underset{x\in\Omega}{\sup} |f(x)|<\infty\}
\end{equation}
}
\vspace{-0.5cm}

equipped with the norm $\|f(x)\|_{L^\infty} = \underset{x\in\Omega}{\sup} |f(x)|$.
\end{definition}

\begin{definition}[Norm of Linear Operators]
Let $X,Y$ be two Banach spaces. The operator $O : X\rightarrow Y$ is linear if and only if $O(\alpha x_1 + \beta x_2) = \alpha O(x_1) + \beta O(x_2)$ for any $\alpha,\beta \in \mathbb{R}$ and $x_1,x_2\in X$. The operator norm of $O$ is defined by $\|O\|_{op} = \underset{x\in X,x\ne 0}{\sup}\frac{\|O(x)\|_Y}{\|x\|_{X}}$.
\end{definition}

\begin{theorem}
Let $B(X;Y)$ be the spaces bounded linear operators, i.e., linear operators with bounded operator norm. Then $B(X;Y)$ is a Banach space. 
\end{theorem}

\begin{proof}
Please refer to the book~\cite{Conw07} for detailed proof.
\end{proof}

The following definitions describe the continuity and differentiability properties of operators between Banach spaces. 

\begin{definition}[Continuity]
Let $X, Y$ be Banach spaces. An operator $g: X\rightarrow Y$ is continuous at $x_0\in X$ if for any $\varepsilon>0$, there exists $\delta>0$ such that for any $x\in X$ satisfying $\|x-x_0\|_X<\delta$, it holds that $\|g(x) - g(x_0)\|_Y<\varepsilon$.
\end{definition}

\begin{definition}\label{def:Frechet}[Fr\'echet Derivative]
Let $X$ and $Y$ be two Banach spaces, and $g$ is an operator from $X$ to $Y$. The operator $g$ is called Fr\'echet differentiable at $x\in X$ if there exists a bounded linear operator $P: X\rightarrow Y$, such that 

\vspace{-0.35cm}
{\small
\begin{equation} \label{eq:FrechetDerivative}
\lim\limits_{\|h\|_X\rightarrow 0}\frac{\|g(x+h)-g(x)-P(x)h\|_Y}{\|h\|_X}\rightarrow 0
\end{equation}
}
\vspace{-0.3cm}

The linear operator $P(x)$ is called the Fr\'echet derivative of $O$ at $x$. 
\end{definition}

\begin{definition}[Nemytskii Operator]
Let $\Omega\subset\mathbb{R}^n$ be a domain and $Y$ be a Banach space. Given a functional $F: \Omega\times X \rightarrow \mathbb{R}$ for any $y\in Y$, a new functional $N(y): \Omega\rightarrow \mathbb{R}$ is defined as $N(y)(x) = F(x;y)$. The operator $N$ is called a Nemytskii operator.
\end{definition}

%% file: Approach.tex
The idea of this work is to develop a  Gaussian parameter smoothing method to transform the base classifiers $f(x;W)$ and $f(x;W_k)$ into their smooth versions $\hat{f}(x;W)$ and $\hat{f}(x;W_k)$ and train an optimal individual smooth classifier $\hat f(x;W_k)$ for each group $G_k$ with only the data regarding $G_k$. By leveraging the theory of nonlinear functional analysis, we reformulate the smooth classifiers $\hat{f}(x;W)$ and $\hat{f}(x;W_k)$ as output functions of a Nemytskii operator $\hat{N}(W)(\cdot)$. We theoretically prove that the Nemytskii operator $\hat{N}(W)$ is smooth and therefore $\hat{N}(\mathbb{R}^m)$ is a Fr\'echet differentiable smooth manifold. We further prove that $\hat{N}(\mathbb{R}^m)$ has a global Lipschitz constant which is independent of the structure of classifiers (e.g., neural networks) and the domain of the input data. Based on global Lipschitz property of $\hat{N}(\mathbb{R}^m)$, we derive the input-agnostic certified group fairness.

First, we use linear functions as an example to demonstrate the concept of Nemytskii operators. In fact,  linear functions can be viewed as  linear functionals from Banach spaces to $\mathbb{R}$, as well as a Nemytskii operator from certain Banach spaces to the space of bounded linear functionals from $L^p$ to $\mathbb{R}$, i.e., $B(L^p)$. Consider a linear function $g(x) = W\cdot x$, where $W\in \mathbb{R}^n$ is a constant coefficient vector and $x\in \mathbb{R}^n$.  If considering the domain of $x$ as a Banach space $(\mathbb{R}^n, \|\cdot \|_{L^p})$, then $g(x)$ can be viewed as a functional from $L^p$ to $\mathbb{R}$. In addition, by employing the H\"older Inequality, we have $|g(x)| = |W\cdot x|\le \|W\|_{L^{p'}}\|x\|_{L^p}$, where $p'$ is the conjugacy of $p$ satisfying $\frac{1}{p}+\frac{1}{p'} = 1$. This implies that the operator norm $\|g \|_{op}$ of the linear functional $g(x)$ is $\|g\|_{op} = \underset{x\in\mathbb{R}^n, x\ne 0}{\sup}\frac{|g(x)|}{\|x\|_{p}} =  \|W\|_{p'}$. Thus, $g(x)$ is a bounded linear functional from $L^p$ to $\mathbb{R}$ for any $W\in L^{p'}$, i.e., $g(x)\in B(L^p)$ for any $W\in L^{p'}$. If we further consider $W$ as a vector variable, the mapping $N(W)$:$W\rightarrow g(x)\in S(L^p)$, i.e., $N(W)(x) = g(x)$, can be viewed as an operator from the Banach space $L^{p'}$ to the Banach space $S(L^p)$. The operator $N(W): L^{p'}\rightarrow S(L^p)$ is a Nemytskii operator.

\vspace{-0.1cm}
\subsection{Fair Classification with Provable Guarantees}\label{sec.Classification}
\vspace{-0.1cm}

In this paper, we define the certified fairness metric as follows. 

\begin{definition}[$\epsilon$-fairness in hypothesis space $\mathcal{H}$]
For $1\le p <\infty$ and a subset $\Omega\subset\mathbb{R}^n$, an overall classifier $h(x;W)\in \mathcal{H}$ for all groups has $\epsilon$-fairness in $L^p(\Omega)$ space if 

\vspace{-0.35cm}
{\small
\begin{equation} \label{eq:Fairness1}
|\Omega_k|^{-1}\|h(x;W) - h(x;W^\ast_k)\|_{{L^p(\Omega_k)} } \le \epsilon,\ k = 1,\cdots,K;
\end{equation}
}
\vspace{-0.3cm}

$h(x;W)$ has $\epsilon$-fairness in $L^\infty(\Omega)$ space if

\vspace{-0.35cm}
{\small
\begin{equation} \label{eq:Fairness2}
\|h(x;W) - h(x;W^\ast_k)\|_{{ L^\infty(\Omega_k)} } \le \epsilon,\ k = 1,\cdots,K
\end{equation}
}
\vspace{-0.3cm}

where $|\Omega_k|$ denotes the cardinality of the subset $\Omega_k$, $h(x;W^\ast_k)$ is the optimal individual classifier for group $G_k$ in $\mathcal{H}$ and $\Omega_k = \{x\in \Omega: z = z_k\}$.
\end{definition}

\begin{definition}[Certified $\epsilon$-fairness in hypothesis space $\mathcal{H}$]
If a classifier $h(x)\in \mathcal{H}$ has $\epsilon$-fairness in $L^p(\Omega)$ for any subset $\Omega\subset\mathbb{R}^n$, then $f(x)$ has certified $\epsilon$-fairness in $L^p$ for $1\le p\le \infty$.
\end{definition}

From the above definition, we know that it is important to choose a suitable hypothesis space for certifying the $\epsilon$-fairness. We define the hypothesis space $\mathcal{H}$ as a Banach manifold in $L^p$ space. We reformulate the base classifiers $f(x;W)$ as output functions of a Nemytskii operator  $N: \mathbb{R}^m \rightarrow L^p$ with $N(W)(x) = f(x;W)$.

We define another Nemytskii operator $\hat{N}: \mathbb{R}^m \rightarrow L^p$ with  Gaussian parameter smoothing as follows.

\vspace{-0.35cm}
{\small
\begin{equation} \label{eq:Nemytskii}
\hat{N}(W)(\cdot) = \mathbb{E}\left(f(\cdot;W+\Delta)\right), \ \Delta \sim \mathcal{N}(0,\sigma^2I)\end{equation}
}
\vspace{-0.3cm}

That is, 

\vspace{-0.35cm}
{\small
\begin{equation} \label{eq:Nemytskii1}
\hat{N}(W)(\cdot)= \frac{1}{(2\pi \sigma^2)^{m/2}}\int_{\mathbb{R}^m}N(W+\Delta)(\cdot)e^{-\frac{\|\Delta\|^2_2}{2\sigma^2}}d\Delta
\end{equation}
}
\vspace{-0.3cm}

\begin{lemma}\label{le:Frechet}
$\hat{N}$ is Fr\'echet differentiable. Therefore, $\hat{N}(\mathbb{R}^m)\subset L^p$ is a smooth manifold. 
\end{lemma}

\begin{proof}
Please refer to Appendix~\ref{sec.Proof} for detailed proof of Lemma~\ref{le:Frechet}.
\end{proof}

We define the hypothesis $\hat{\mathcal{H}}$ by $\hat{\mathcal{H}}= \hat{N}(\mathbb{R}^m)(\cdot)$. It is clear that $\hat{\mathcal{H}}$ is not a linear vector space, because $\hat{N}(W_1)(\cdot) + \hat{N}(W_2)(\cdot)$ may not be in $\hat{\mathcal{H}}$. However, $\hat{\mathcal{H}}$ is a smooth manifold which guarantees that the parameters $W$ are trainable. Our goal is to certify $\epsilon$-fairness in $\hat{\mathcal{H}}$ for the best possible $\epsilon$. The key idea is that the smooth manifold $\hat{\mathcal{H}}$ has a global Lipschitz constant which is independent of the structure of classifiers (e.g., neural networks) and the domain of the input data, which is demonstrated in lemma \ref{le:Nemytskii1} below. In contrast, the manifold $N(\mathbb{R}^m)(\cdot)$ without Gaussian parameter smoothing is only locally Lipschitz with respect to the parameter $W$ and the Lipschitz constant is determined by the input data and the network structures. In addition, the Lipschitz constant regarding $W$ could be rather large since it is hard to control the amplification of difference through propagation over neural networks and thus the Lipschitz constant keeps increasing with the number of layers, which prohibits the certification of group fairness. 

\begin{lemma}\label{le:Nemytskii1}
For any $W_1, W_2 \in \mathbb{R}^m$ satisfying $\|W_1 - W_2\|_2 \leq d$, it holds the that

\vspace{-0.35cm}
{\small
\begin{equation} \label{eq:Nemytskii1}
\begin{split}
&|\Omega|^{-1}\|\hat{N}(W_1)(x) - \hat{N}(W_2)(x)\|_{L^p(\Omega)} \le \frac{d}{\sqrt{2\pi}\sigma}\\
& \text{if}\; 1\le p<\infty, \\
&\|\hat{N}(W_1)(x) - \hat{N}(W_2)(x)\|_{L^\infty(\Omega)} \le \frac{d}{\sqrt{2\pi}\sigma}
\end{split}
\end{equation}
}
\vspace{-0.3cm}

for any $\Omega\subset\mathbb{R}^n$.
\end{lemma}

\begin{proof}
Please refer to Appendix~\ref{sec.Proof} for detailed proof of Lemma~\ref{le:Nemytskii1}.
\end{proof} 

According to the conclusion of Lemma~\ref{le:Nemytskii1}, the following theorem derives the certified fairness by our proposed input-agnostic certified group fairness method and demonstrates that the certified fairness is irrelevant to the domain of the input data.

\begin{theorem}\label{th:Certify}
Let $\hat{N}(W^\ast_k)$ be the optimal individual classifier for group $G_k$ in the hypothesis space $\hat{\mathcal{H}}$ $(k = 1,2,\cdots,K)$. Let $W^\ast = \frac{W^\ast_1+\cdots+W^\ast_K}{K}$ and $\hat{N}(W^\ast)$ be the overall fair classifier. If $\underset{1\leq k,l \leq K}{\max}{\|W_k - W_l\|_2} = d$, then for any $\Omega\subset\mathbb{R}^n$, it holds that 

\vspace{-0.35cm}
{\small
\begin{equation} \label{eq:Nemytskii1}
\begin{split}
&|\Omega_k|^{-1}\|\hat{N}(W^\ast)(x) - \hat{N}(W^\ast_k)(x)\|_{L^p(\Omega)} \le \frac{(K-1)d}{\sqrt{2\pi}K\sigma} \\
&\text{if}\;1\le p <\infty, \\
&\|\hat{N}(W^\ast)(x) - \hat{N}(W^\ast_k)(x)\|_{L^\infty(\Omega_k)} \le \frac{(K-1)d}{\sqrt{2\pi}K\sigma}
\end{split}
\end{equation}
}
\vspace{-0.3cm}

Therefore, $\hat{N}(W^\ast)(x)$ has certified $\frac{(K-1)d}{\sqrt{2\pi}K\sigma}$-fairness in $L^p$ space for any $1\le p\le \infty$.
\end{theorem}

\begin{proof}
Please refer to Appendix~\ref{sec.Proof} for detailed proof of Lemma~\ref{th:Certify}.
\end{proof}

Notice that Theorem \ref{th:Certify} is satisfied for any $\Omega\subset\mathbb{R}^n$. Therefore, it provides an input-agnostic certified group fairness guarantee.

In Theorem \ref{th:Certify}, the standard deviation $\sigma$ in Gaussian noise serves as a tradeoff hyperparameter to well balance the fairness and accuracy achieved by the smoothed classifier. A larger $\sigma$ results in better classification fairness, while a smaller $\sigma$ leads to better accuracy. Especially, when $\sigma\rightarrow 0$, the smooth classifier $\hat{N}(W)(x)$ converges to the base classifier $N(W)(x)$, which is verified by Lemma~\ref{le:Balance}.

\begin{theorem}[Dominated Convergence Theorem]
Let $f_n(x)$ be a sequence of measurable functions. Suppose $f_n(x)\rightarrow f(x)$ almost everywhere and there exists a integrable function $g(x)$ such that $|f_n(x)|\le g(x)$ for any $n$ and $x$. Then 

\vspace{-0.35cm}
{\small
\begin{equation}
\lim\limits_{n\rightarrow\infty} \int f_n(x) d\mu = \int f(x) d\mu
\end{equation}
}
\vspace{-0.5cm}
\end{theorem}

\begin{proof}
Please refer to the book~\cite{Conw07} for detailed proof.
\end{proof}

We will leverage Lebesgue's Dominated Convergence Theorem to prove the following lemma.

\begin{lemma}\label{le:Balance}
For any $x\in \mathbb{R}^n$, $\hat{N}(W)(x)\rightarrow N(W)(x)$ as $\sigma\rightarrow 0$. 
\end{lemma}

\begin{proof}
Please refer to Appendix~\ref{sec.Proof} for detailed proof of Lemma~\ref{le:Balance}.
\end{proof}

Based on Theorem \ref{th:Certify}, we modify the training objective of fair classification as follows.

\vspace{-0.35cm}
{\small
\begin{equation} \label{eq:Loss}
\begin{split}
\min_{W_1,\cdots,W_K} &\sum_{k=1}^{K} \underset{(x_i, z_i, y_i) \in G_k} {\mathbb{E}}\Big(\mathcal{L}\big(\hat{N}(W_k)(x_i), y_i\big)\Big) + \\
\alpha &\sum_{k=1}^{K} \sum_{l>k}^{K} \left\|W_k-W_l\right\|^2_{2}
\end{split}
\end{equation}
}
\vspace{-0.5cm}

After the training, the overall fair classifier $\hat{N}(W^\ast)(x)$ for all groups is generated by averaging the parameters of the smooth classifier $\hat{N}(W^\ast_k)(x)$ for each group $G_k$, i.e., $W^\ast = \frac{W^\ast_1+\cdots+W^\ast_K}{K}$. $\alpha$ is a hyperparameter to balance the fairness and accuracy by $\hat{N}(W^\ast)(x)$.

By reducing the loss function to local chart of the manifold $\hat{\mathcal{H}}$, i.e. the parameter space, we can view the loss function as a function of $W_k$. By changing the notation slightly, the loss function can be rewritten as follows. 

\vspace{-0.35cm}
{\small
\begin{equation} \label{eq:Loss2}
\min_{W_1,\cdots,W_K} \sum_{k=1}^{K} \mathcal{L}_k\big(W_k\big) + \alpha \sum_{k=1}^{K} \sum_{l>k}^{K} \left\|W_k-W_l\right\|^2_{2}
\end{equation}
}
\vspace{-0.5cm}

The training process of each individual classifier is given below.

\vspace{-0.35cm}
{\small
\begin{equation} \label{eq:Training}
\begin{split}
&W_k(t+1) \\
= &W_k(t) - \eta \nabla \mathcal{L}_k(W_k(t)) - \\
&2\alpha \eta \Big(\sum_{l=1}^{k-1} (W_k(t) - W_l(t)) + \sum_{l=k+1}^{K} (W_k(t) - W_l(t))\Big)
\end{split}
\end{equation}
}
\vspace{-0.5cm}

where $\eta$ is the learning rate and $t$ is the iteration step.

\vspace{-0.1cm}
\subsection{Computation of Numerical Integrals}\label{sec.Integral}
\vspace{-0.1cm}

Since the dimension of the parameter space of the classifiers, especially deep neural networks, is often pretty high, it is not practical to use deterministic numerical integration methods due to the high computational cost. In our implementation, we utilize the Monte-Carlo integration technique to compute the numerical integrals.

\vspace{-0.35cm}
{\small
\begin{equation} \label{eq:Integral}
\hat{N}(W)(\cdot)= \frac{1}{(2\pi \sigma^2)^{m/2}}\int_{\mathbb{R}^m}N(W+\Delta)(\cdot)e^{-\frac{\|\Delta\|^2_2}{2\sigma^2}}d\Delta
\end{equation}
}
\vspace{-0.5cm}

Specifically, we sample $\Delta_j \sim \mathcal{N}(0, \sigma^{2} I)$, $j=1,\cdots, 10^5$, i.e., the sample size is $100,000$, where $\mathcal{N}(0, \sigma^{2} I)$ is the isotropic Gaussian distribution. We then calculate the following Monte-Carlo integration.

\vspace{-0.35cm}
{\small
\begin{equation} \label{eq:MC}
\hat{N}(W)(x) \approx 10^{-5}\sum_{j = 1}^{10^5}  N(W+\Delta_j)(x)
\end{equation}
}
\vspace{-0.5cm}

By using the standard error analysis of Monte-Carlo integration, the probability of the error of the above Monte-Carlo integral lying between the following range is equal to $\frac{1}{2\pi}\int_{-a}^a e^{-t^2/2}dt$.

\vspace{-0.35cm}
{\small
\begin{equation} \label{eq:Error}
\big(-a \mathbb{V}(N(W+\Delta)(x)\big)/\sqrt{10^5}, a \mathbb{V}(N(W+\Delta)(x))/\sqrt{10^5}\big)
\end{equation}
}
\vspace{-0.5cm}

where $\mathbb{V}(N(W+\Delta)(x))$ is the variance of $N(W+\Delta)(x)$ for $\Delta \sim \mathcal{N}(0, \sigma^{2} I)$. 

Since the classifier $|N(W)(x)|\le 1$ for any $W$ and $x$, $\mathbb{V}(N(W+\Delta)(x))\le 1$. Taking $a = 3$ as an example in Eq.\eqref{eq:Error}, we obtain that the probability of the error of the numerical integral in Eq.\eqref{eq:MC} being less than $0.01$ is about $99.7\%$. This implies that our input-agnostic certified group fairness technique has great potential to achieve the superior certified fairness guarantees for any input data, which is desirable in practice.

%% file: Experiment.tex
In this section, we have evaluated the certified fairness of our \method\ model and other comparison methods over two benchmark datasets from the fair classification: Adult from the UCI Machine Learning Repository~\cite{UCI} and COMPAS introduced by the ProPublica~\cite{ProPublica}. Both datasets have binary labels and a mixture of numerical and categorical features. The Adult dataset contains 48,842 individuals with 18 binary features and a label indicating whether the income is greater than 50,000 USD or not. The COMPAS dataset consists of 6,172 individuals with 10 binary features and a label that takes value 1 if the individual does not reoffend and 0 otherwise. In both datasets, We use {\it Sex} and {\it Race} as the protected attributes respectively. The experiments exactly follow the same settings described by other group fairness methods with provable guarantees~\cite{ULP19,KJWC21,CHKV21,IKL21}. The datasets are preprocessed by using AIF360 toolkit~\cite{EDHH18}. 

{\bf Baselines.} We compare the \method\ model with ten state-of-the-art group fairness models, including 
six regular group fairness and three provable group fairness guarantee approaches.
{\bf ApxFair} is a classifier that are fair not only with respect to the training distribution, but also for a class of distributions that are weighted perturbations of the training samples~\cite{MDJW20}.
{\bf ARL} is an adversarially reweighted learning model that aims to improve the utility for worst-off protected groups, without access to protected features at training or inference time~\cite{LBCL20}.
{\bf Fair Mixup} is a data augmentation strategy that improves the generalization of group fairness metrics~\cite{ChMr21}.
{\bf FairBatch} adaptively selects minibatch sizes for the purpose of improving model fairness based on various prominent fairness measures~\cite{RLWS21a}.
{\bf fair-robust-selection (FRS)} is a sample selection-based algorithm for fair and robust training by solving a combinatorial optimization problem for the unbiased selection of samples~\cite{RLWS21b}.
{\bf Implicit} derives an implicit path alignment algorithm to learn a fair representation by encouraging the invariant optimal predictors on the top of data representation~\cite{Anon22c}.
{\bf Group-Fair} is a meta-algorithm for provable group fairness guarantees by approximately reducing classification problems with general types of fairness constraints to ones with convex constraints~\cite{CHKV19}.
{\bf FCRL} provides theoretical guarantees on the parity of downstream classifiers by limiting the mutual information between representations and protected attributes~\cite{GFDS21}.
{\bf DLR} is an optimization framework for learning a fair classifier with provable guarantees in the presence of noisy perturbations in the protected attributes.~\cite{CHKV21}.
To our best knowledge, this work is the first to certify the group fairness of classifiers with theoretical input-agnostic guarantees, while there is no need to know the shift between training and deployment datasets with respect to sensitive attributes.

{\bf Variants of \method\ model.} We evaluate three versions of \method\ to show the strengths of different techniques. By leveraging the theory of nonlinear functional analysis, \method\ utilizes our proposed Gaussian parameter smoothing method with parameter disparity term to produce a fair classifier with input-agnostic certified group fairness guarantees. \methodI\ makes use of the Gaussian parameter smoothing only, while \methodII\ employs the parameter disparity term only.

\begin{table}[t]\addtolength{\tabcolsep}{-4pt}
\caption{Accuracy and Fairness with {\it Sex} Attribute}
\vspace{-0.2cm}
\small
\begin{center}
\hspace{-0.8cm}
\begin{tabular}{l|ccc|ccc}
\hline {\bf Dataset} & \multicolumn{3}{|c} \textbf{Adult \qquad \ } & \multicolumn{3}{|c} \textbf{COMPAS \quad \ } \\
\hline {\bf Protected Attribute} & \multicolumn{3}{|c} \textbf{Sex \qquad \ \ \ } & \multicolumn{3}{|c} \textbf{Sex \qquad \ \ } \\
\hline {\bf Metric} & {\bf Acc.} & {\bf $\Delta_{DP}$} & {\bf $\Delta_{EO}$} & {\bf Acc.} & {\bf $\Delta_{DP}$} & {\bf $\Delta_{EO}$} \\
\hline ApxFair & 0.7 & 0.24 & 0.18 & 0.52 & 0.21 & 0.58 \\
ARL & {\bf 0.84} & 0.13 & 0.1 & 0.49 & 0.17 & 0.28 \\
Fair Mixup & 0.76 & {\bf 0.05} & 0.27 & 0.66 & 0.17 & 0.29 \\
FairBatch & {\bf 0.84} & 0.19 & 0.19 & 0.65 & 0.34 & 0.64 \\
FRS & 0.44 & 0.11 & 0.16 & 0.31 & 0.13 & {\bf 0.19} \\
Implicit & 0.8 & 0.11 & 0.15 & 0.64 & 0.17 & 0.3 \\
Group-Fair & 0.76 & 0.33 & 0.51 & 0.44 & 0.13 & 0.25 \\
FCRL & 0.81 & {\bf 0.05} & 0.22 & 0.56 & 0.19 & 0.39 \\
DLR & 0.5 & 0.13 & 0.22 & 0.49 & 0.28 & 0.46 \\
\hline \method\ & {\bf 0.84} & {\bf 0.05} & {\bf 0.08} & {\bf 0.67} & {\bf 0.11} & {\bf 0.19} \\
\hline
\end{tabular}
\label{tbl.Sex}
\end{center}
\vspace{-0.1cm}
\end{table}

{\bf Evaluation metrics.} We use two popular measures in fair machine learning to verify the fairness of different methods: {\bf Demographic Parity Disparity ($\Delta_{DP}$)} and {\bf Equalized Odds Disparity ($\Delta_{EO}$)}~\cite{CHS20,ChMr21,RLWS21,RLWS21,GFDS21,PCBZ21,Anon22b}. Demographic Parity defines fairness as an equal probability of being classified with the positive label, i.e., each group has the same probability of being classified with the positive outcome~\cite{CaHa20}. Equalized Odds defines fairness as an equal probability if the true positive rate and false positive rate of a classifier is the same across different groups~\cite{CaHa20}. $\Delta_{DP}$ or $\Delta_{EO}$ represent the difference of Demographic Parity or Equalized Odds between sensitive groups respectively. A smaller $\Delta_{DP}$ or $\Delta_{EO}$ score indicates a better fairness. In addition, we use {\bf Accuracy} to evaluate the quality of the fair classification algorithms.

{\bf Performance with {\it Sex} as protected attribute.} Table \ref{tbl.Sex} exhibits the $Accuracy$, $\Delta_{DP}$, and $\Delta_{EO}$ scores of ten fair classification algorithms over two datasets of Adult and COMPAS.  
It is observed that among ten fair classification methods our \method\ method achieve the highest $Accuracy$, the smallest $\Delta_{DP}$, and the lowest $\Delta_{EO}$ on two datasets in most experiments, showing the superior quality and fairness of \method\ against protected attributes. Compared to the fair classification results by other models, \method, on average, achieves 18.7\% $Accuracy$ boost, 48.2\% $\Delta_{DP}$ improvement, and 49.1\% $\Delta_{EO}$ decrease on two datasets. In addition, the promising performance of \method\ over both datasets implies that \method\ has great potential as a general fair classification solution to other datasets, which is desirable in practice.

\begin{figure}[t]
\mbox{
\hspace{-1.4cm}
\subfigure[Accuracy]{\epsfig{figure=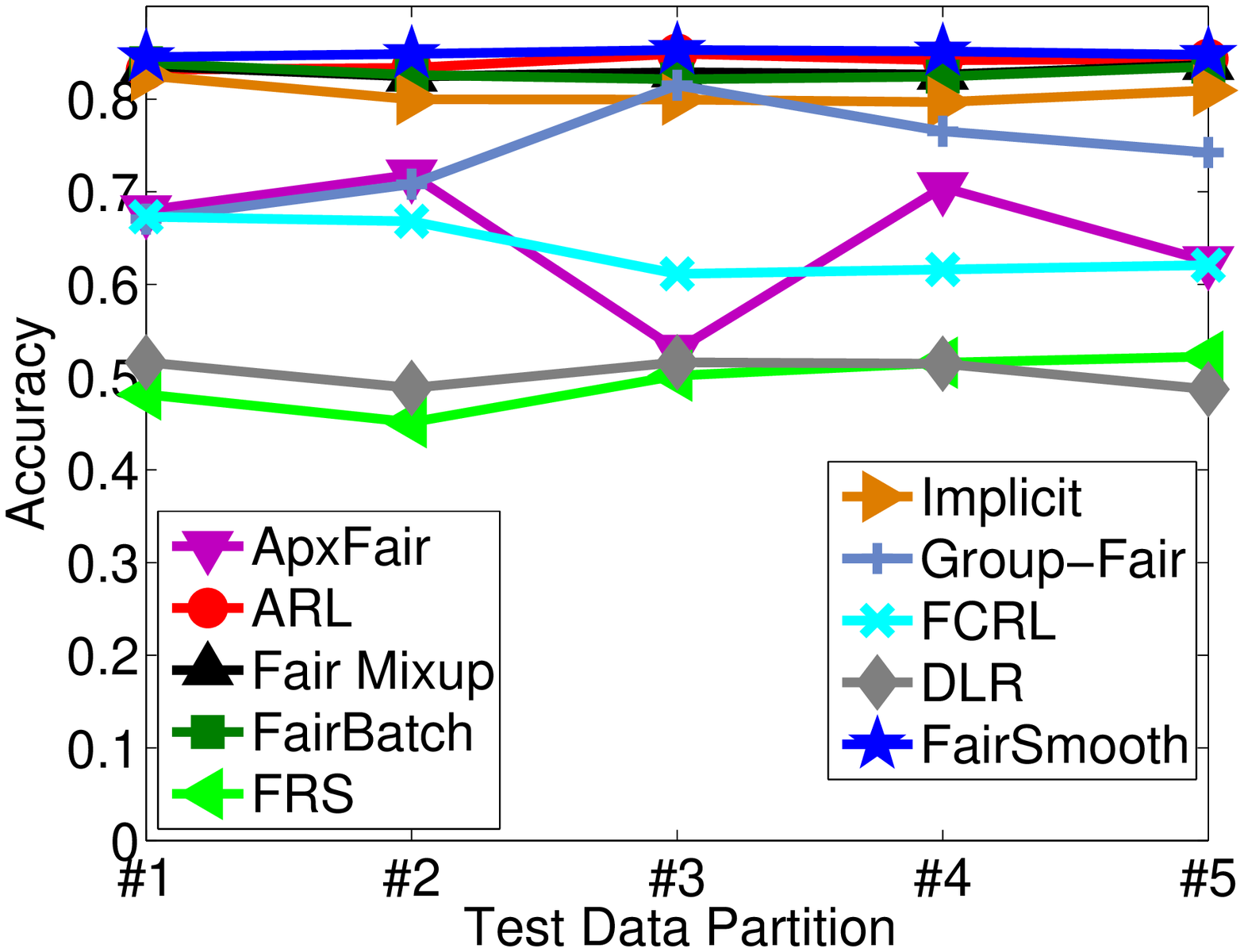, height=1.4in, width=0.395\linewidth}} \hspace{-0.225cm}
\subfigure[$\Delta_{DP}$]{\epsfig{figure=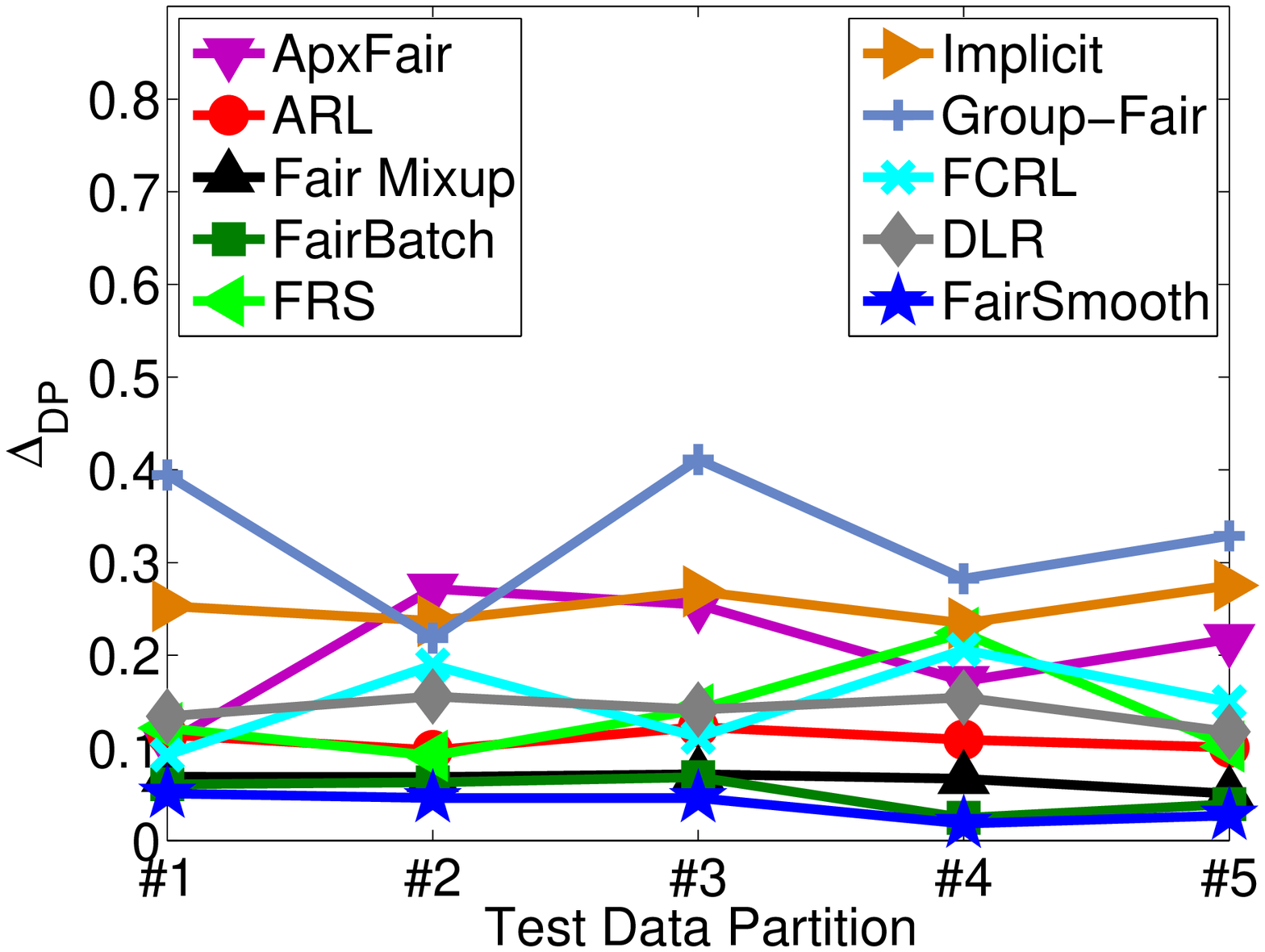, height=1.4in, width=0.395\linewidth}} \hspace{-0.225cm}
\subfigure[$\Delta_{EO}$]{\epsfig{figure=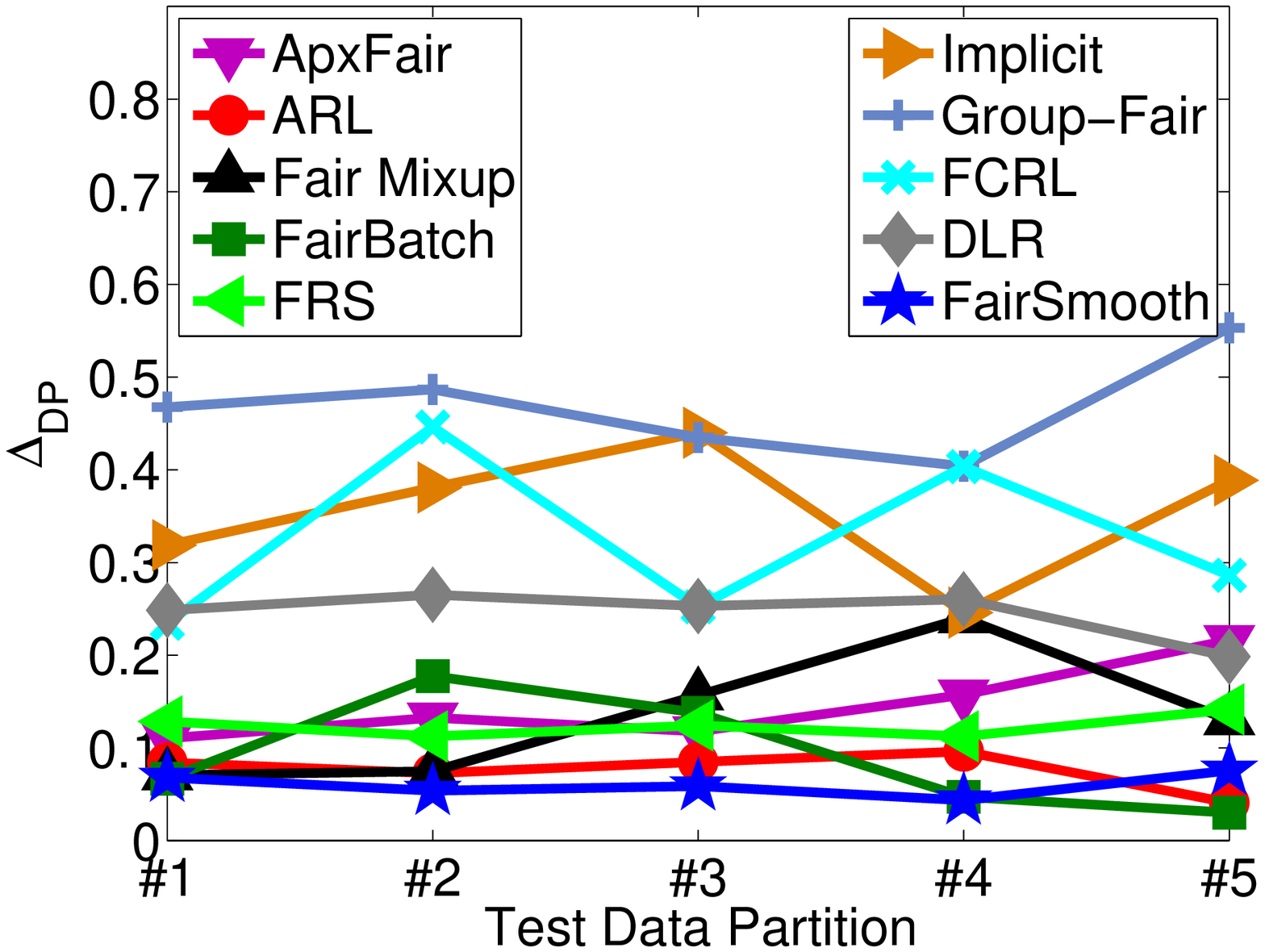, height=1.4in, width=0.395\linewidth}}}
\vspace{-0.2cm}
\caption{Performance with Varying Test Data Partitions on Adult}
\label{fig.DataPartition}
\vspace{-0.3cm}
\end{figure}

{\bf Performance with varying test data partitions on Adult.} Figures~\ref{fig.DataPartition} presents the fair classification performance by varying the distributions of test data on Adult. We randomly divide the test data into five partitions, which sample 5\%, 10\%, 15\%, 30\%, and 40\% of the entire test data respectively. Therefore, these five partitions have quite different distributions from each other. In addition, we allow the same samples to occur in multiple partitions.
It is obvious that the performance curves by other classification algorithms oscillate up and down with varying test data partitions. This phenomenon indicates that not only the Accuracy but also the Fairness of current fair classification methods are sensitive to test data distributions. However, the curves by our \method\ keep relatively stable. This demonstrates that our \method\ is able to produce the input-agnostic fair classification results, without the knowledge or assumption about the distribution of training and deployment data. 
In addition, \method\ still achieves the highest $Accuracy$ values ($>$ 0.83), the smallest $\Delta_{DP}$ scores ($<$ 0.11), and the lowest $\Delta_{EO}$ values ($<$ 0.10), which are better than other nine methods in most tests.

{\bf Ablation study.} Table \ref{tbl.Variant} presents the performance of fair classification on two datasets with three variants of our \method\ model. We observe the complete \method\ achieves the highest $Accuracy$ values ($>$ 0.66), the smallest $\Delta_{DP}$ scores ($<$ 0.05), and the lowest $\Delta_{EO}$ values ($<$ 0.08), which are obviously better than other versions. Compared with \methodI, \method\ performs well in most experiments. A reasonable explanation is that the parameter disparity term forces all individual classifiers to approach each other, which helps recover the most accurate model for each group as well as produce more fair overall classifier. In addition, \method\ achieves the better performance than \methodII. A rational guess is that the parameter disparity term only fails to provide a tight fairness guarantee without help of Gaussian parameter smoothing.

\begin{table}[t]\addtolength{\tabcolsep}{-4pt}
\caption{Accuracy and Fairness of \method\ variants}
\vspace{-0.2cm}
\begin{center}
\begin{tabular}{l|ccc|ccc}
\hline {\bf Dataset} & \multicolumn{3}{|c} \textbf{Adult \qquad \ } & \multicolumn{3}{|c} \textbf{COMPAS \quad \ } \\
\hline {\bf Protected Attribute} & \multicolumn{3}{|c} \textbf{Sex \qquad \ \ \ } & \multicolumn{3}{|c} \textbf{Race \qquad \ \ } \\
\hline {\bf Metric} & {\bf Acc.} & {\bf $\Delta_{DP}$} & {\bf $\Delta_{EO}$} & {\bf Acc.} & {\bf $\Delta_{DP}$} & {\bf $\Delta_{EO}$} \\
\hline \methodI\ & {\bf 0.84} & 0.2 & 0.22 & 0.59 & 0.15 & 0.25 \\
\methodII\ & 0.81 & 0.27 & 0.27 & 0.56 & 0.07 & 0.1 \\
\hline \method\ & {\bf 0.84} & {\bf 0.05} & {\bf 0.08} & {\bf 0.66} & {\bf 0.03} & {\bf 0.01} \\
\hline
\end{tabular}
\label{tbl.Variant}
\end{center}
\vspace{-0.1cm}
\end{table}

\begin{figure}[t]
\mbox{
\subfigure[Adult]{\epsfig{figure=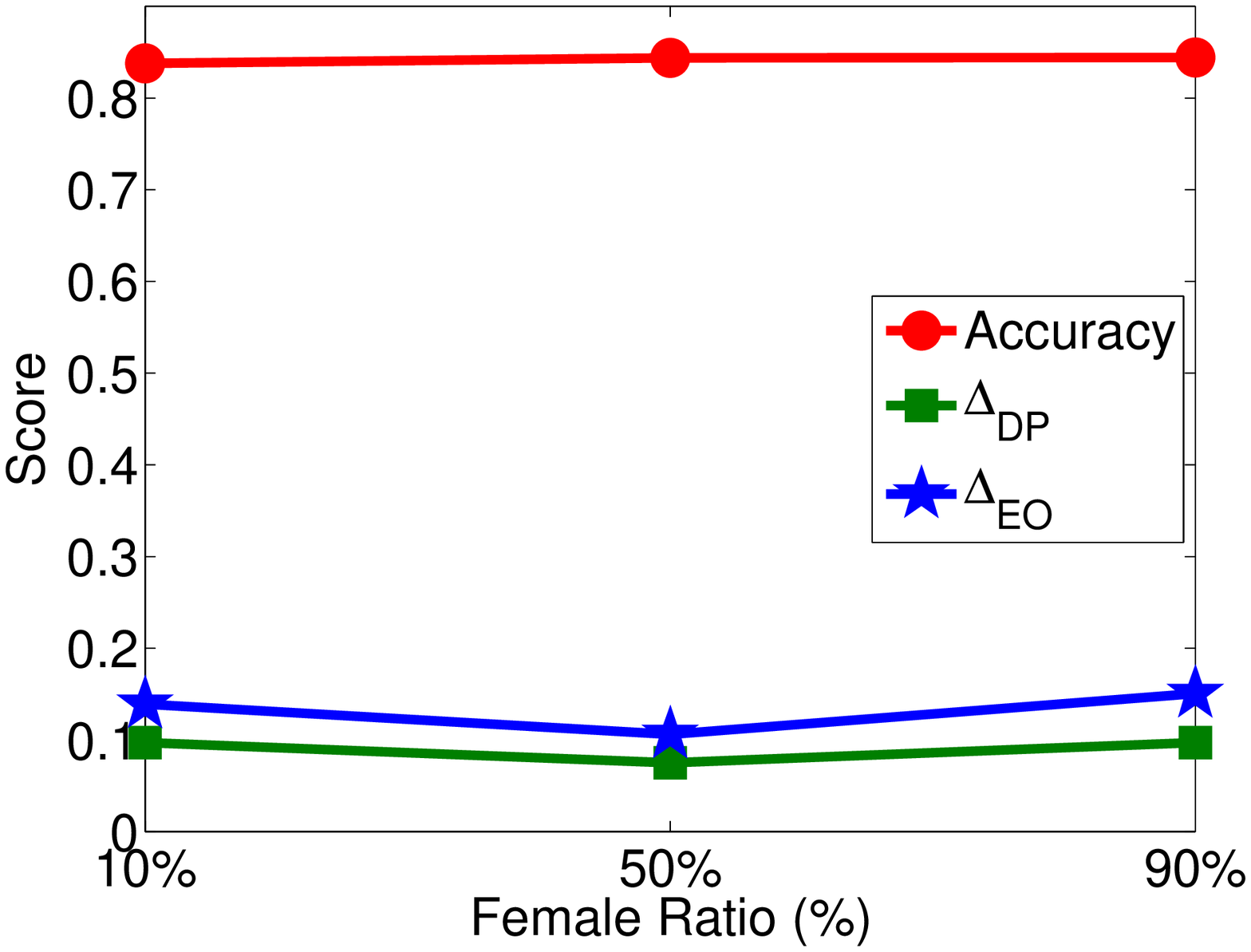, height=1.4in, width=0.495\linewidth}} \hspace{-0.225cm}
\subfigure[COMPAS]{\epsfig{figure=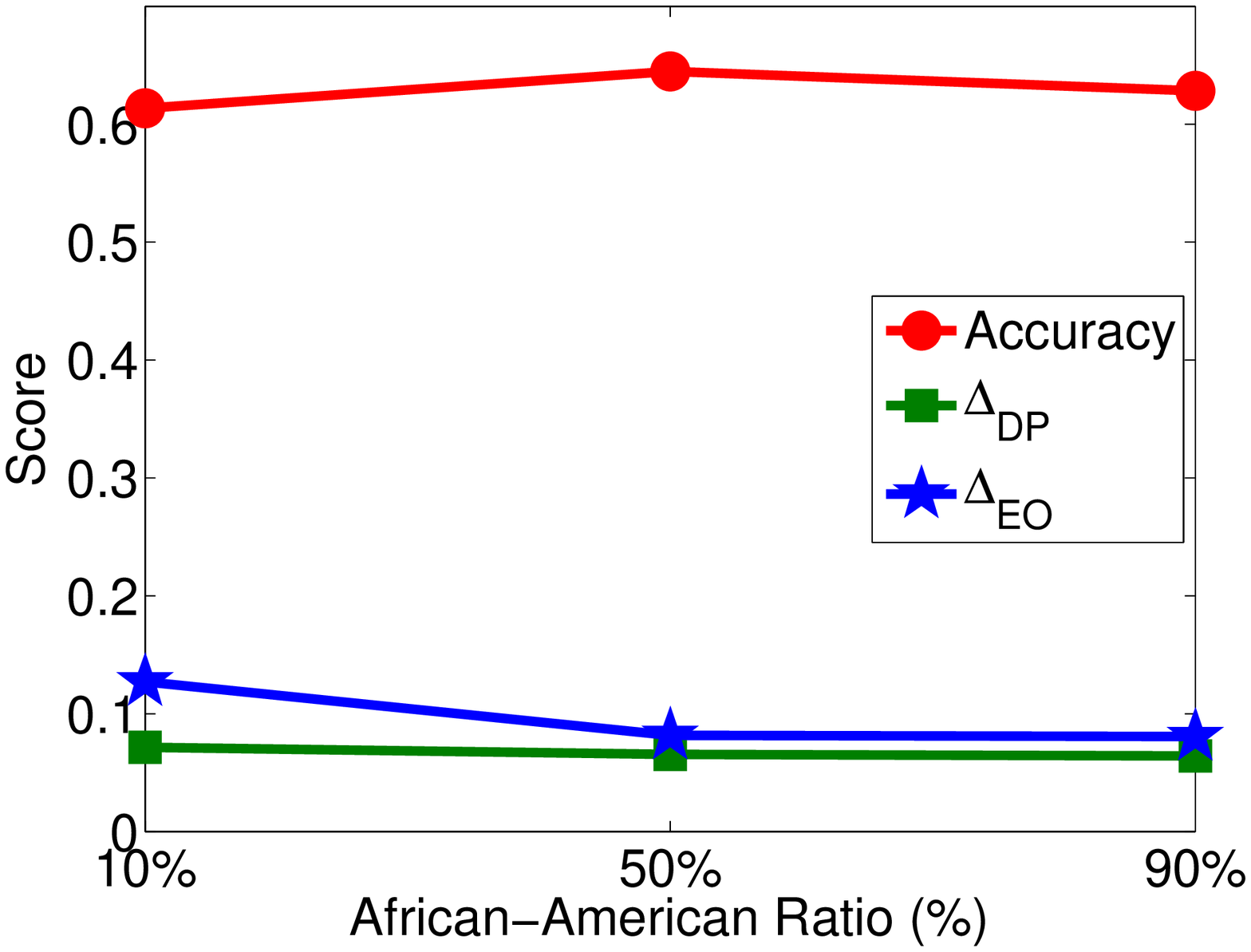, height=1.4in, width=0.495\linewidth}}}
\vspace{-0.2cm}
\caption{Performance with Varying Group Ratios}
\label{fig.GroupRatio}
\vspace{-0.3cm}
\end{figure}

{\bf Impact of group ratios.} Figure~\ref{fig.GroupRatio} shows the quality of ten fair classification algorithms on Adult (and COMPAS) with {\it Sex} (and {\it Race}) as protected attribute by varying the {\it Female} (and {\it African-American}) ratio of training data from 10\% to 90\% respectively. We randomly reduce the number of {\it Female} (or {\it African-American}) training samples for decreasing the ratio or reduce the number of {\it Male} (or {\it Other}) training samples for increasing the ratio.
We average metrics over 10 trials of randomly selected groups for imbalanced experiments. The manually imbalanced datasets are used to benchmark fair classification performance.
We make the following observations on the performances by our fair classification algorithm: (1) The accuracy (or fairness) curves initially keep increasing (or fairness) and then slightly drop (or raise) when the ratio increases; (2) Although we see that there is a little oscillation in two figures, the performance curves still keep relative stable. A reasonable explanation is that the group imbalance degree of training data definitely affects the performance of our fair classification method. However, our \method\ method is still robust to imbalanced training data with the help of Gaussian parameter smoothing and parameter disparity term.

{\bf Impact of $\sigma$ and $\alpha$ on Adult.} Figures~\ref{fig.Parameter} (a) and (b) show the impact of the standard deviation $\sigma$ in Gaussian distribution and the weight $\alpha$ of parameter disparity term in our \method\ model over the Adult dataset. The performance curves initially improve and finally become worse when two parameters increase. This demonstrates that there must exist an optimal $\sigma$ and $\alpha$ for the Gaussian parameter smoothing and the parameter disparity term. A too large $\sigma$ may introduce too much Gaussian noise and thus it is hard to guarantee the performance of Gaussian parameter smoothing, while a too small $\sigma$ may loss the strength of Gaussian parameter smoothing. Similarly, a too large $\alpha$ may make all individual classifiers become the same and thus fails to ensure the classification accuracy, while a too small $\alpha$ may make all individual classifiers far away from each other and thus fails to provide the fairness guarantees. Thus, it is important to choose the appropriate parameters for well balancing the accuracy and fairness.

{\bf Fairness certificate computation.} Table \ref{tbl.Distance} presents the average distance between the parameters of the overall and individual classifiers on Adult and COMPAS. We also measure the fairness certificates $\frac{(K-1)d}{\sqrt{2\pi}K\sigma}$ over two datasets. In the current experiments, the calculated average distances and fairness certificates on two datasets are very small with order $10^{-3}$ and $10^{-2}$ respectively. This implies that introducing the Gaussian parameter smoothing method and parameter disparity term to the training objective is able to help the overall classifier achieve high accuracy, which providing good fairness guarantees.

\begin{table}[t]\addtolength{\tabcolsep}{-4pt}
\caption{Fairness Certificate}
\vspace{-0.2cm}
\begin{center}
\begin{tabular}{l|c|c}
\hline {\bf Dataset} & \textbf{Adult} & \textbf{COMPAS} \\
\hline {\bf Average Distance $d$} & 0.0074 & 0.0057 \\
\hline {\bf Fairness Certificate $\frac{(K-1)d}{\sqrt{2\pi}K\sigma}$} & 0.0147 & 0.0114 \\
\hline
\end{tabular}
\label{tbl.Distance}
\end{center}
\vspace{-0.1cm}
\end{table}

\begin{figure}[t]
\mbox{
\subfigure[$\sigma$]{\epsfig{figure=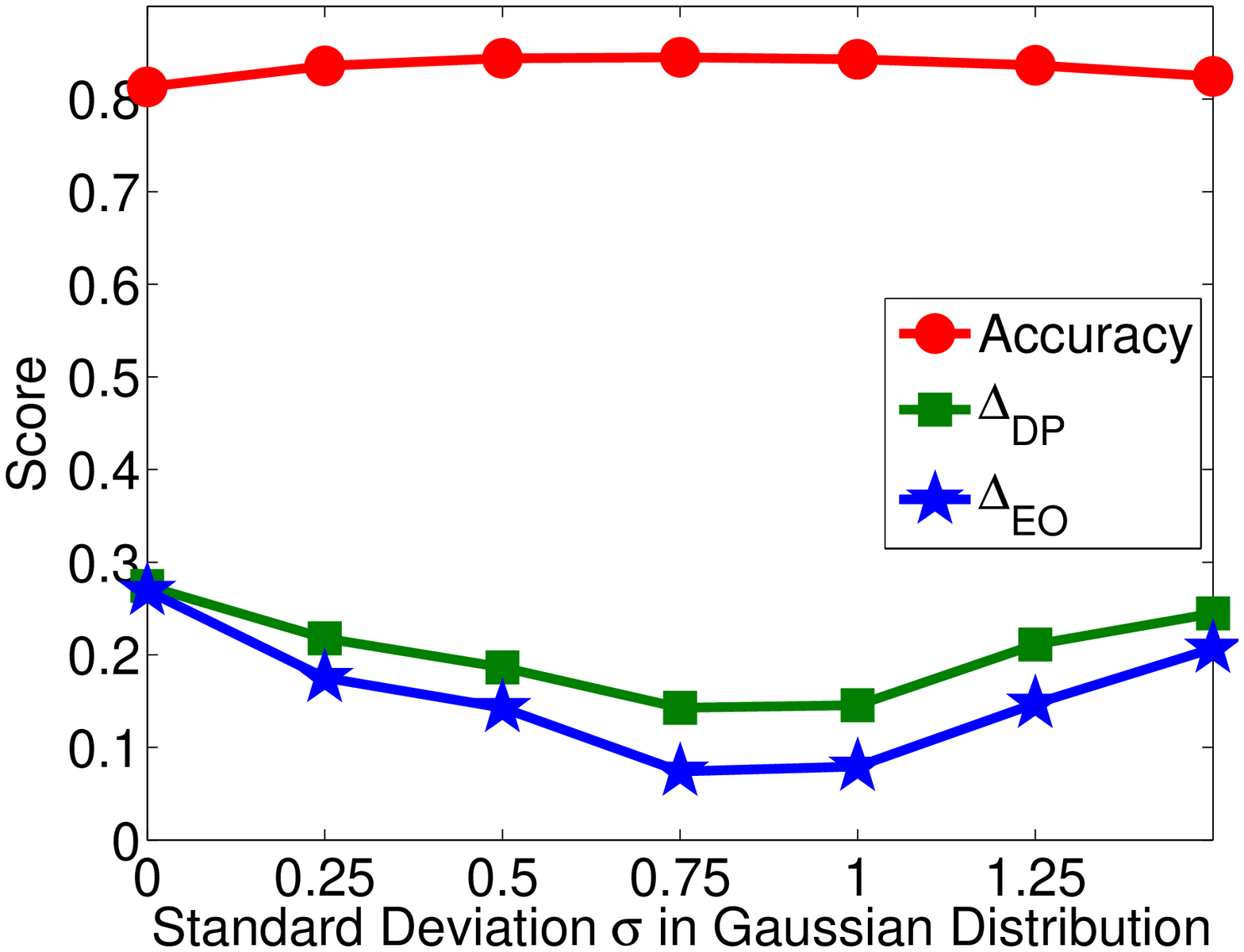, height=1.4in, width=0.495\linewidth}} \hspace{-0.225cm}
\subfigure[$\alpha$]{\epsfig{figure=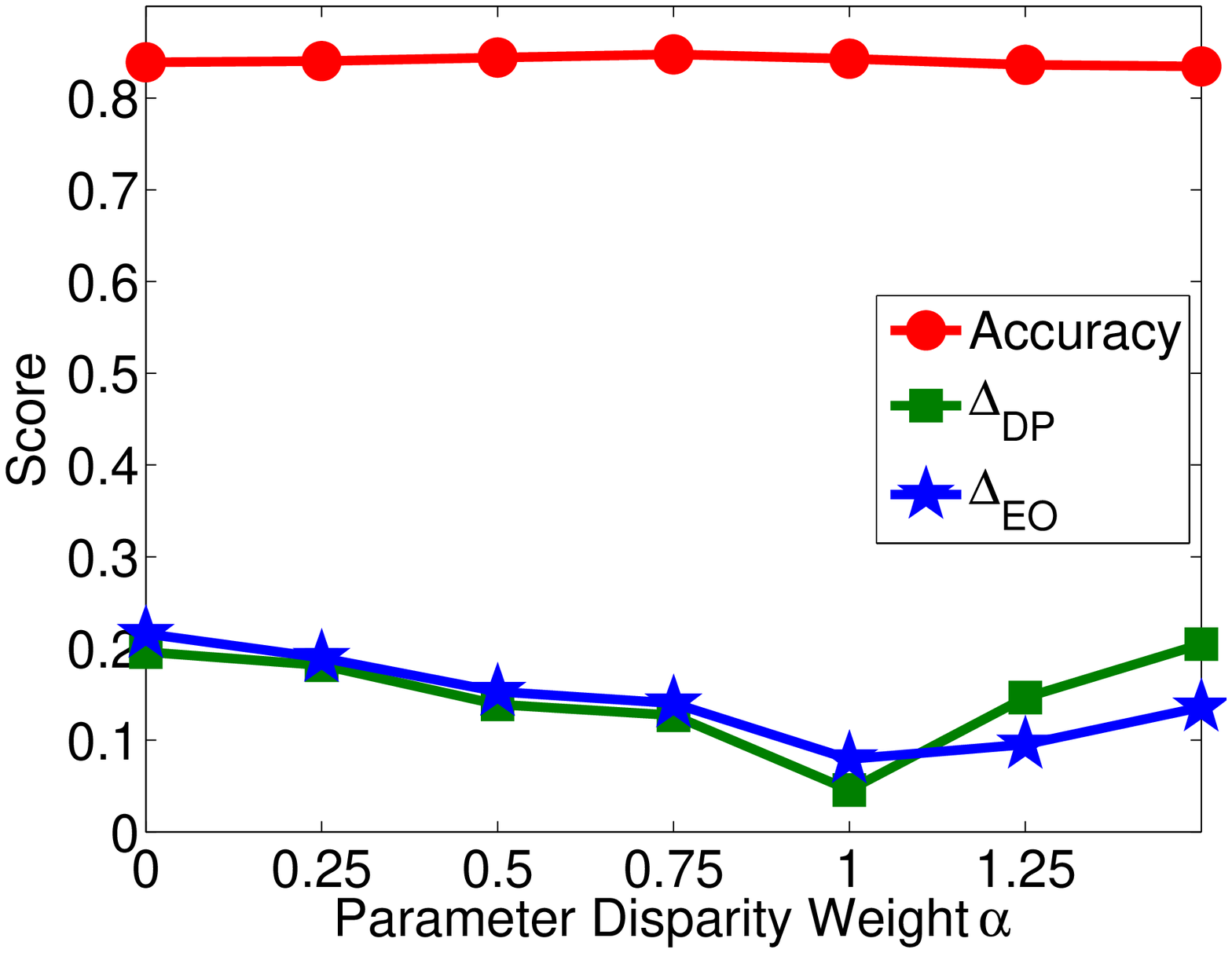, height=1.4in, width=0.495\linewidth}}}
\vspace{-0.2cm}
\caption{Performance with Varying Parameters on Adult}
\label{fig.Parameter}
\vspace{-0.3cm}
\end{figure}

%% file: Related.tex
{\bf Fair Classification.} 
Trustworthy machine learning has attracted active research in recent years~\cite{PLZW18,LWCD19,ZRDJ20,ZZZS20,ZhLi19,WLZZ20b,ZZWS21,ZZZW21a,RZJZ21,ZRJZ21,MWDW21,ZZZW21}.
Fair classification techniques aim to guarantee the learnt classifiers that are not only accurate but also fair with respect to sensitive attributes~\cite{CaHa20,MMSL21}.
Existing techniques on fair classification can be broadly classified into three categories based on which stage of the machine learning pipeline they target: 
(1) Pre-processing techniques try to pre-process the original dataset for reducing bias present in the training data and achieve an unbiased dataset for learning~\cite{KaCa11,TRT11,ZWSP13,FFMS15,EdSt16,MOW17,KXPK18,MCPZ18,SKGZ19,CiKo19,GFDS21}. The pre-processing methods are beneficial when the onus of fairness is on a third party or the data controller~\cite{MOW17,CiKo19,GFDS21}. One of the main objectives of the pre-processing methods is to ensure that any downstream classifiers can freely use the sanitized datasets for fair machine learning~\cite{EdSt16,MOW17,SKGZ19,MCPZ18}. However, their operationalization often leads to an approach that may not ensure fairness guarantees~\cite{GFDS21};
(2) In-processing approaches aim to regularize the classifier during training to ensure the learning fairness (e.g. Zafar et al. (2017)). The in-processing algorithms can be categorized into two core mechanisms. 
Penalty-based methods add  constraints~\cite{HPS16,ZVGG17a,ZVGG17b,ZLM18,NaSh18,GYF18,MeWi18,Nara18,KRR18,CHKV19,CKLV19,NMS19,CJGW19} or a regularizer~\cite{KAAS12,GYF18,HFGK18,AAV19,HuVi19,JPSJ19,MDJW20,SHV20,ChMr21} to the learning objective that penalize or prevent parameter choices that lead to unfair decisions.
Adversarial learning-based methods train an adversary in parallel to the classifier that tries to predict the protected attribute from the model outputs~\cite{BCZC17,WVP18,ZLM18,LBCL20}. If this cannot be done better than chance level, the fairness is achieved.
However, many of the aforementioned algorithms do not come with theoretical fairness guarantees; perhaps because the manually-crafted heuristics are used to generate fair classifiers and the resulting optimization problem is non-convex;
(3) Post-processing algorithms formulate fair classification problems as constrained optimization problems: first train an unconstrained optimal classifier and then shift the classifiers' decision boundary according to the fairness requirement~\cite{FKL16,HPS16,GCGF16,WGOS17,PRWK17,DIKL18,CDHO20}. This is a simple, reliable and often effective method, but it requires the protected attribute to be available at prediction time.

{\bf Group Fairness with Provable Guarantees.} Group-based fair learning models aim to reduce the performance gaps between different protected groups with sensitive attributes~\cite{ZVGG17a,ZVGG17b,GYF18,KXPK18,MeWi18,MDJW20,WGNC20,YCK20,ZTLL20,LBCL20,PST21,CHKV21,GDL21,ChMr21,RLWS21,SGJ21,QPLH21,LiWa21,MMYS21,CMV21,DMWT21,RLWS21,Anon22a,ZLL22}.
However, most of existing studies fail to provide a provable fairness guarantee, because the manually-crafted heuristics are used to generate fair classifiers and the resulting optimization problems are non-convex~\cite{ZVGG17a,ZVGG17b,KXPK18}.
Only recently, researchers attempt to theoretically provide different classification algorithms with provable fairness guarantees~\cite{FFMS15,MOW17,CHKV19,ULP19,GFDS21,KJWC21,SAPB21,CHKV21,IKL21,Anon22b}.

Most of the above fairness guarantees are based on an assumption that the training and deployment data follow the same distribution. However, this assumption is false for many real-world problems~\cite{ZQDX21}. A recent study reports that the models based on the above assumption often violate these guarantees and exhibit unfair bias when evaluated on data from a different distribution~\cite{Anon22b}. Shifty is the first method to provide provable guarantees under a demographic shift between training and deployment data~\cite{Anon22b}. However, the fact that the old and new demographic proportions should be known in Shifty, which limits the applicability of Shifty in real-world scenarios.

To our best knowledge, this work is the first to certify the group fairness of classifiers with theoretical input-agnostic guarantees, while there is no need to know the shift between training and deployment datasets with respect to sensitive attributes, by leveraging the theory of nonlinear functional analysis, including Nemytskii operator and smooth manifold.

%% file: Appendix.tex

\subsection{Theoretical Proof}\label{sec.Proof}

\setcounter{lemma}{0}
\setcounter{theorem}{1}

\begin{lemma}\label{le:Frechet}
$\hat{N}$ is Fr\'echet differentiable. Therefore, $\hat{N}(\mathbb{R}^m)\subset L^p$ is a smooth manifold. 
\end{lemma}

\begin{proof}
Recall the definition of Fr\'echet derivative in Definition \ref{def:Frechet}. Generally speaking, the proof of an operator being Fr\'echet differentiable consists of two steps: 1) Find a candidate $P(x)$ for the Fr\'echet derivative; 2) Show that $L(x)$ satisfies the limit condition in the definition. We first show $\hat N(W)$ is Fr\'echet differentiable. Without loss of generality, we prove it for $\sigma =1$. 
By the change of variable, we rewrite $\hat\phi(W)$ as 

\begin{equation}
\hat{N}(W)(\cdot)= \frac{1}{(2\pi)^{m/2}}\int_{\mathbb{R}^m}N(\Delta)(\cdot)e^{-\frac{\|W-\Delta\|^2_2}{2}}d\Delta
\end{equation}

We define a linear operator $P(W)(\delta): \mathbb{R}^m \rightarrow L^p$ by 

\begin{equation}
P(W)(\delta)(\cdot)= -\frac{1}{(2\pi)^{m/2}}\int_{\mathbb{R}^m}N(\Delta)(\cdot)e^{-\frac{\|W - \Delta\|^2_2}{2}}(W-\Delta)\cdot\delta d\Delta
\end{equation}

where $P(W)$ is the candidate for the Fr\'echet derivative of $\hat N(W)$. 

Next, we show that $P(W)$ is indeed the Fr\'echet derivative. We compute 
\begin{equation}\label{S1}
\begin{split}
&\|\hat{N}(W+\delta)(\cdot) - \hat{N}(W)(\cdot) - P(W)\delta(\cdot)\|_{L^p(\Omega)}\\
= &\Big\|\frac{1}{(2\pi)^{m/2}}\int_{\mathbb{R}^m}N(\Delta)(\cdot)\left(e^{-\frac{\|W+\delta-\Delta\|^2_2}{2}} - e^{-\frac{\|W-\Delta\|^2_2}{2}} + e^{-\frac{\|W - \Delta\|^2_2}{2}}(W-\Delta)\cdot\delta \right)d\Delta\Big\|_{L^p(\Omega)}
\end{split}
\end{equation}
By direct calculations, we have
\begin{equation}\label{exp-diff}
\frac{d}{ds}e^{-\frac{\|W+s\delta-\Delta\|^2_2}{2}} =- e^{-\frac{\|W+s\delta-\Delta\|^2_2}{2}}(W+s\delta-\Delta)\cdot\delta
\end{equation}

By integrating the above identity from $0$ to $1$, we obtain
\begin{equation}\label{exp-diff}
e^{-\frac{\|W+\delta-\Delta\|^2_2}{2}} - e^{-\frac{\|W-\Delta\|^2_2}{2}} = -\int_0^1 e^{-\frac{\|W+s\delta-\Delta\|^2_2}{2}}(W+s\delta-\Delta)\cdot\delta ds
\end{equation}

By plugging Eq.\eqref{exp-diff} into Eq.\eqref{S1}, we have
\begin{equation}\label{S2}
\begin{split}
&\|\hat{N}(W+\delta)(\cdot) - \hat{N}(W)(\cdot) - P(W)\delta(\cdot)\|_{L^p(\Omega)}\\
= &\Big\|\frac{1}{(2\pi)^{m/2}}\int_{\mathbb{R}^m}N(\Delta)(\cdot)\left(\int_0^1 e^{-\frac{\|W+s\delta-\Delta\|^2_2}{2}}(W+s\delta-\Delta)\cdot\delta - e^{-\frac{\|W - \Delta\|^2_2}{2}}(W-\Delta)\cdot\delta  ds\right)d\Delta\Big\|_{L^p(\Omega)}\\
= &\Big\|\frac{1}{(2\pi)^{m/2}}\int_{\mathbb{R}^m}N(\Delta)(\cdot)\left(\int_0^1( e^{-\frac{\|W+s\delta-\Delta\|^2_2}{2}} - e^{-\frac{\|W - \Delta\|^2_2}{2}})(W-\Delta)\cdot\delta + e^{-\frac{\|W +s\delta- \Delta\|^2_2}{2}}s\delta^2  ds\right)d\Delta\Big\|_{L^p(\Omega)}
\end{split}
\end{equation}

Based on the Fubini's Theorem, we derive that 
\begin{equation}
\begin{split}
&\|\hat{N}(W+\delta)(\cdot) - \hat{N}(W)(\cdot) - P(W)\delta(\cdot)\|_{L^p(\Omega)}\\
= &\Big\|\frac{1}{(2\pi)^{m/2}}\int_0^1\left(\int_{\mathbb{R}^m}N(\Delta)(\cdot)( e^{-\frac{\|W+s\delta-\Delta\|^2_2}{2}} - e^{-\frac{\|W - \Delta\|^2_2}{2}})(W-\Delta)\cdot\delta + e^{-\frac{\|W +s\delta- \Delta\|^2_2}{2}}s\delta^2  d\Delta\right)ds\Big\|_{L^p(\Omega)}
\end{split}
\end{equation}

It is clear that 
\begin{equation}
\frac{1}{(2\pi)^{m/2}}\int_{\mathbb{R}^m} e^{-\frac{\|W +s\delta- \Delta\|^2_2}{2}}\delta^2  d\Delta = \frac{1}{(2\pi)^{m/2}}\int_{\mathbb{R}^m} e^{-\frac{\| \Delta\|^2_2}{2}}\delta^2  d\Delta = \|\delta\|_2^2
\end{equation}

Notice that $|N(\Delta)(x)|\le 1$ for any $\Delta$ and $x$. Therefore, 
\begin{equation}\label{o1}
\Big\|\frac{1}{(2\pi)^{m/2}}\int_0^1\int_{\mathbb{R}^m} N(\Delta)(\cdot)e^{-\frac{\|W +s\delta- \Delta\|^2_2}{2}}s\delta^2  d\Delta ds\Big\|_{L^p(\Omega)} \le \|1\|_{L^p(\Omega)} \|\delta\|_2^2
\end{equation}

Moreover, similar to Eq.\eqref{exp-diff}, one can check that
\begin{equation}\label{exp-diff2}
e^{-\frac{\|W+s\delta-\Delta\|^2_2}{2}} - e^{-\frac{\|W-\Delta\|^2_2}{2}} =- \int_0^1 e^{-\frac{\|W+st\delta-\Delta\|^2_2}{2}}(W+s\delta-\Delta)\cdot\delta dt
\end{equation}

It follows that
\begin{equation}
\begin{split}
&\Big\|\frac{1}{(2\pi)^{m/2}}\int_0^1\int_{\mathbb{R}^m}N(\Delta)(\cdot)( e^{-\frac{\|W+s\delta-\Delta\|^2_2}{2}} - e^{-\frac{\|W - \Delta\|^2_2}{2}})(W-\Delta)\cdot\delta d\Delta ds \Big\|_{L^p(\Omega)}\\
= &\Big\|\frac{1}{(2\pi)^{m/2}}\int_0^1\int_0^1\int_{\mathbb{R}^m}N(\Delta)(\cdot)e^{-\frac{\|W+st\delta-\Delta\|^2_2}{2}}(W+s\delta-\Delta)\cdot\delta(W-\Delta)\cdot\delta d\Delta ds dt \Big\|_{L^p(\Omega)}\\
\le & \Big\|\frac{1}{(2\pi)^{m/2}}\int_0^1\int_0^1\int_{\mathbb{R}^m}e^{-\frac{\|W+st\delta-\Delta\|^2_2}{2}}\|W+s\delta-\Delta\|_2\|W-\Delta\|_2\|\delta\|^2 d\Delta ds dt \Big\|_{L^p(\Omega)}\\
\end{split}
\end{equation}

Based on Dominant Convergence Theorem, as $\delta \rightarrow 0$, we have

\begin{equation}
\int_{\mathbb{R}^m}e^{-\frac{\|W+st\delta-\Delta\|^2_2}{2}}\|W+s\delta-\Delta\|_2\|W-\Delta\|_2\|\delta\|^2 d\Delta \rightarrow \int_{\mathbb{R}^m}e^{-\frac{\|W-\Delta\|^2_2}{2}}\|W-\Delta\|^2_2 d\Delta
\end{equation}

Therefore, as $\delta \rightarrow 0$, we have
\begin{equation}\label{o2}
\begin{split}
&\Big\|\frac{1}{(2\pi)^{m/2}}\int_0^1\int_{\mathbb{R}^m}N(\Delta)(\cdot)( e^{-\frac{\|W+s\delta-\Delta\|^2_2}{2}} - e^{-\frac{\|W - \Delta\|^2_2}{2}})(W-\Delta)\cdot\delta d\Delta ds \Big\|_{L^p(\Omega)}\\
\le & \Big\|\frac{1}{(2\pi)^{m/2}}\int_{\mathbb{R}^m}e^{-\frac{\|W'\|^2_2}{2}}\|W'\|^2_2 dW'\Big\|_{L^p(\Omega)}\|\delta\|^2
\end{split}
\end{equation}

Combining Eq.\eqref{o1} and Eq.\eqref{o2}, we obtain 
\begin{equation}
\hat{N}(W+\delta)(\cdot) - \hat{N}(W)(\cdot) - P(W)\delta(\cdot)\|_{L^p(\Omega)} = O(\delta^2)
\end{equation}

This implies that $P(W)$ is Fr\'echet derivative of $\hat N$. By similar arguments, one can show that $\hat N(W)$ is infinitely Fr\'echet differentiable. Note that $\hat N (W)$ can be viewed as a global chart for the manifold $\hat N(\mathbb{R}^m)$. Therefore, $\hat N(\mathbb{R}^m)$ is a smooth manifold.
\end{proof}

\begin{lemma}\label{le:Nemytskii1}
For any $W_1, W_2 \in \mathbb{R}^m$ satisfying $\|W_1 - W_2\|_2 \leq d$, it holds the that

\vspace{-0.35cm}
\begin{equation} \label{eq:Nemytskii1}
\begin{split}
&|\Omega|^{-1}\|\hat{N}(W_1)(x) - \hat{N}(W_2)(x)\|_{L^p(\Omega)} \le \frac{d}{\sqrt{2\pi}\sigma}\\
& \text{if}\; 1\le p<\infty, \\
&\|\hat{N}(W_1)(x) - \hat{N}(W_2)(x)\|_{L^\infty(\Omega)} \le \frac{d}{\sqrt{2\pi}\sigma}
\end{split}
\end{equation}
\vspace{-0.3cm}

for any $\Omega\subset\mathbb{R}^n$.
\end{lemma}

\begin{proof}
Since $\mathcal{N}(0,\sigma^2I) = \sigma^{-m}\mathcal{N}(0,I)$, it suffices to prove the lemma for $\sigma = 1$. Rewrite $\hat\phi(W)$ as 

\begin{equation}
\hat{N}(W)(\cdot)= \frac{1}{(2\pi)^{m/2}}\int_{\mathbb{R}^m}N(\Delta)(\cdot)e^{-\frac{\|W-\Delta\|^2_2}{2}}d\Delta
\end{equation}

By direct calculations, it holds that 
\begin{equation}\label{Lip1}
\begin{split}
&\|\hat{N}(W_1)(\cdot) - \hat{N}(W_2)(\cdot) \|_{L^p(\Omega)}\\
= &\Big\|\frac{1}{(2\pi)^{m/2}}\int_{\mathbb{R}^m}N(\Delta)(\cdot)\left(e^{-\frac{\|W_1-\Delta\|^2_2}{2}} - e^{-\frac{\|W_2-\Delta\|^2_2}{2}} \right)d\Delta\Big\|_{L^p(\Omega)}
\end{split}
\end{equation}

One can check that

\begin{equation}
\frac{d}{ds}e^{-\frac{\|W_2-\Delta+s(W_1 - W_2)\|^2_2}{2}} =- e^{-\frac{\|W_2-\Delta+s(W_1 - W_2)\|^2_2}{2}}(W_2-\Delta+s(W_1 - W_2))\cdot (W_1 -W_2)
\end{equation}

By integrating the above identity from $0$ to $1$, we obtain

\begin{equation}\label{Lip-exp-diff}
\begin{split}
e^{-\frac{\|W_1-\Delta\|^2_2}{2}} - e^{-\frac{\|W_2 - \Delta\|^2_2}{2}} = -&\int_0^1 e^{-\frac{\|W_2 - \Delta +s(W_1 - W_2)\|^2_2}{2}}(W_2 - \Delta+s(W_1 - W_2))\cdot (W_1 -W_2) ds
\end{split}
\end{equation}

By plugging Eq.\eqref{Lip-exp-diff} into Eq.\eqref{Lip1}, we have
\begin{equation}
\begin{split}
&\|\hat{N}(W_1)(\cdot) - \hat{N}(W_2)(\cdot) \|_{L^p(\Omega)}\\
= &\Big\|\frac{1}{(2\pi)^{m/2}}\int_{\mathbb{R}^m}N(\Delta)(\cdot)\left(\int_0^1 e^{-\frac{\|W_2-\Delta+s(W_1 - W_2)\|^2_2}{2}}(W_2-\Delta+s(W_1 - W_2))\cdot (W_1 -W_2)ds\right)d\Delta\Big\|_{L^p(\Omega)}\\
= &\Big\|\frac{1}{(2\pi)^{m/2}}\int_{\mathbb{R}^m}\int_0^1 N(\Delta)(\cdot)e^{-\frac{\|W_2 - \Delta +s(W_1 - W_2)\|^2_2}{2}}(W_2 - \Delta +s(W_1 - W_2))\cdot (W_1 -W_2)dsd\Delta\Big\|_{L^p(\Omega)}
\end{split}
\end{equation}

By applying Fubini's Theorem to the above inequality, we have
\begin{equation}
\begin{split}
&\|\hat{N}(W_1)(\cdot) - \hat{N}(W_2)(\cdot) \|_{L^p(\Omega)}\\
= &\Big\|\frac{1}{(2\pi)^{m/2}}\int_0^1\int_{\mathbb{R}^m} N(\Delta)(\cdot)e^{-\frac{\|W_2 - \Delta +s(W_1 - W_2)\|^2_2}{2}}(W_2 - \Delta +s(W_1 - W_2))\cdot (W_1 -W_2)d\Delta ds\Big\|_{L^p(\Omega)}
\end{split}
\end{equation}
For convenience, we let $\widetilde{W} = W_2  +s(W_1 - W_2)$ and $\bar W = W_1 - W_2$. Then, we rewrite
\begin{equation}\label{Lip2}
\begin{split}
&\|\hat{N}(W_1)(\cdot) - \hat{N}(W_2)(\cdot) \|_{L^p(\Omega)}\\
= &\Big\|\frac{1}{(2\pi)^{m/2}}\int_0^1\int_{\mathbb{R}^m} N(\Delta)(\cdot)e^{\frac{-\|\widetilde{W} - \Delta\|^2_2}{2}}(\widetilde{W} - \Delta)\cdot \bar W d\Delta ds\Big\|_{L^p(\Omega)}\\
= & \Big\|\frac{1}{(2\pi)^{m/2}}\int_0^1\int_{\mathbb{R}^m} N(\Delta)(\cdot)e^{\frac{-\|\widetilde{W} - \Delta\|^2_2}{2}}(\widetilde{W} - \Delta)\cdot \frac{\bar W}{\|\bar W\|_2} d\Delta ds\Big\|_{L^p(\Omega)}\|\bar W\|_2\\
\end{split}
\end{equation}

Notice that $0\le N(W)(x)\le1$ for any $W$ and $x$ and $\frac{\bar W}{\|\bar W\|_2}$ is a unit vector. By using standard Gaussian integrals, for any unit vector $u$, we obtain 

\begin{equation}
\begin{split}
& \Big\|\frac{1}{(2\pi)^{m/2}}\int_0^1\int_{\mathbb{R}^m} N(\Delta)(\cdot)e^{\frac{-\|\widetilde{W} - \Delta\|^2_2}{2}}(\widetilde{W} - \Delta)\cdot u d\Delta ds\Big\|_{L^p(\Omega)}\\
\le & \Big\|\frac{1}{(2\pi)^{m/2}}\int_0^1\int_{\mathbb{R}^m} e^{\frac{-\|\widetilde{W} - \Delta\|^2_2}{2}}|(\widetilde{W} - \Delta)\cdot u| d\Delta ds\Big\|_{L^p(\Omega)}\\
= &  \Big\|\frac{1}{(2\pi)^{m/2}}\int_0^1\int_{\mathbb{R}^m} e^{\frac{-\|\Delta\|^2_2}{2}}|\Delta\cdot u| d\Delta ds\Big\|_{L^p(\Omega)}\\
 \end{split}
\end{equation}

We rotate the coordinate system such that the unit vector $u$ coincides with any coordinate axis. Let $u, v_1,\cdots, v_{m-1}$ be the standard basis of the rotated coordinate system. In other words, for any $\Delta \in \mathbb{R}^m$, we decompose $\Delta$ into

\begin{equation}
\Delta = (\Delta\cdot u) u +\sum_{l=1}^{m-1}(\Delta\cdot v_{m-1})v_{m-1}
\end{equation}

Since the Jacobian of rotation is $1$, by utilizing standard Gaussian integrals, one can check that
\begin{equation}
\begin{split}
\int_{\mathbb{R}^m} e^{\frac{-\|\Delta\|^2_2}{2}}|\Delta\cdot u| d\Delta
= &\int_{\mathbb{R}^m} e^{\frac{-\sum_{l=1}^{m-1}|\Delta\cdot v_l|^2_2}{2}}e^{\frac{-|\Delta\cdot u|^2_2}{2}}|\Delta\cdot u| d\Delta\\
= &\int_{\mathbb{R}^m} e^{\frac{-\sum_{l=1}^{m-1}\tau_l^2}{2}}e^{\frac{-|\tau_m|^2}{2}}d\tau_1\cdots d\tau_m = 2(2\pi)^{\frac{m-1}{2}}
\end{split}
\end{equation}

Consequently, we have 
\begin{equation}
\Big\|\frac{1}{(2\pi)^{m/2}}\int_0^1\int_{\mathbb{R}^m} N(\Delta)(\cdot)e^{\frac{-\|\widetilde{W} - \Delta\|^2_2}{2}}(\widetilde{W} - \Delta)\cdot u d\Delta ds\Big\|_{L^p(\Omega)} = \| \frac{1}{(2\pi)^{1/2}}\|_{L^p(\Omega)}
\end{equation}

Therefore, it follows that for $1\le p <\infty$,
\begin{equation}
\|\hat{N}(W_1)(\cdot) - \hat{N}(W_2)(\cdot) \|_{L^p(\Omega)} \le  \frac{1}{(2\pi)^{1/2}}|\Omega| \|W_1-W_2\|_2
\end{equation}
\begin{equation}
\|\hat{N}(W_1)(\cdot) - \hat{N}(W_2)(\cdot) \|_{L^\infty(\Omega)} \le  \frac{1}{(2\pi)^{1/2}}\|W_1-W_2\|_2
\end{equation}
The proof is finished. 
\end{proof}

\begin{theorem}\label{th:CertifyA}
Let $\hat{N}(W^\ast_k)$ be the optimal individual classifier for group $G_k$ in the hypothesis space $\hat{\mathcal{H}}$ $(k = 1,2,\cdots,K)$. Let $W^\ast = \frac{W^\ast_1+\cdots+W^\ast_K}{K}$ and $\hat{N}(W^\ast)$ be the overall fair classifier. If $\underset{1\leq k,l \leq K}{\max}{\|W_k - W_l\|_2} = d$, then for any $\Omega\subset\mathbb{R}^n$, it holds that 

\vspace{-0.35cm}
\begin{equation} \label{eq:Nemytskii1}
\begin{split}
&|\Omega_k|^{-1}\|\hat{N}(W^\ast)(x) - \hat{N}(W^\ast_k)(x)\|_{L^p(\Omega)} \le \frac{(K-1)d}{\sqrt{2\pi}K\sigma} \\
&\text{if}\;1\le p <\infty, \\
&\|\hat{N}(W^\ast)(x) - \hat{N}(W^\ast_k)(x)\|_{L^\infty(\Omega_k)} \le \frac{(K-1)d}{\sqrt{2\pi}K\sigma}
\end{split}
\end{equation}
\vspace{-0.3cm}

Therefore, $\hat{N}(W^\ast)(x)$ has certified $\frac{(K-1)d}{\sqrt{2\pi}K\sigma}$-fairness in $L^p$ for any $1\le p\le \infty$.
\end{theorem}

\begin{proof}
By direct calculations, we have that for any $k = 1,\cdots K$, $W^\ast - W^\ast_k = \frac{\sum_{l\ne k} W^\ast - W^\ast_k}{K}$, which implies that $\|W^\ast - W^\ast_k\|\le \frac{(K-1)d}{K}$. Then the desired results follows directly from Lemma \ref{le:Nemytskii1}. 
\end{proof}

\begin{lemma}\label{le:Balance}
For any $x\in \mathbb{R}^n$, $\hat{N}(W)(x)\rightarrow N(W)(x)$ as $\sigma\rightarrow 0$. 
\end{lemma}

\begin{proof}
Through a change of variable $\Delta = \sigma \tau$, we rewrite 

\begin{equation}
\hat{N}(W)(x) = \frac{1}{(2\pi)^{m/2}}\int_{\mathbb{R}^m}N(W+\sigma\tau)(x)e^{-\frac{\|\tau|^2_2}{2}}d\tau
\end{equation}

It is clear that $|N(W+\sigma\tau)(x)e^{-\frac{\|\tau|^2_2}{2}}|\le e^{-\frac{\|\tau|^2_2}{2}}$ and $ e^{-\frac{\|\tau|^2_2}{2}}$ is integrable. Moreover, by the continuity of $N(W+\sigma\tau)(x)$, one has $N(W+\sigma\tau)(x)\rightarrow N(W)(x)$ as $\sigma\rightarrow 0$. Therefore, by Dominated Convergence Theorem, we obtain that as $\sigma \rightarrow 0$, it holds 

\begin{equation}
\frac{1}{(2\pi)^{m/2}}\int_{\mathbb{R}^m}N(W+\sigma\tau)(x)e^{-\frac{\|\tau|^2_2}{2}}d\tau\rightarrow \frac{1}{(2\pi)^{m/2}}\int_{\mathbb{R}^m}N(W)(x)e^{-\frac{\|\tau|^2_2}{2}}d\tau
\end{equation}

This completes the proof.
\end{proof}

\subsection{Additional Experiments}\label{sec.AdditionalExperiments}

In this section, we conduct more experiments to validate the accuracy and fairness of our proposed \method method and evaluate the sensitivity of various parameters in our input-agnostic certified group fairness framework for the fair classification task.

\begin{table}[t]\addtolength{\tabcolsep}{-4pt}
\caption{Accuracy and Fairness with {\it Race} Attribute}
\vspace{-0.2cm}
\begin{center}
\begin{tabular}{l|ccc|ccc}
\hline {\bf Dataset} & \multicolumn{3}{|c} \textbf{Adult \qquad \ } & \multicolumn{3}{|c} \textbf{COMPAS \quad \ } \\
\hline {\bf Protected Attribute} & \multicolumn{3}{|c} \textbf{Race \qquad \ \ } & \multicolumn{3}{|c} \textbf{Race \qquad \ \ } \\
\hline {\bf Metric} & {\bf Acc.} & {\bf $\Delta_{DP}$} & {\bf $\Delta_{EO}$} & {\bf Acc.} & {\bf $\Delta_{DP}$} & {\bf $\Delta_{EO}$} \\
\hline ApxFair & 0.72 & 0.21 & 1.19 & 0.6 & 0.17 & 0.22 \\
ARL & {\bf 0.84} & 0.18 & 0.19 & 0.5 & 0.05 & 0.17 \\
Fair Mixup & 0.83 & 0.09 & 0.14 & 0.65 & {\bf 0.03} & 0.05 \\
FairBatch & {\bf 0.84} & 0.15 & 0.35 & 0.65 & 0.31 & 0.59 \\
FRS & 0.65 & 0.15 & 0.17 & 0.31 & 0.04 & 0.13 \\
Implicit & 0.67 & 0.18 & 0.2 & 0.6 & 0.11 & 0.3 \\
Group-Fair & 0.73 & 0.12 & {\bf 0.12} & 0.49 & 0.10 & 0.15 \\
FCRL & 0.81 & 0.12 & 0.16 & 0.49 & 0.08 & 0.16 \\
DLR & 0.76 & 0.09 & 0.18 & 0.41 & 0.12 & 0.21 \\
\hline \method\ & {\bf 0.84} & {\bf 0.04} & 0.13 & {\bf 0.66} & {\bf 0.03} & {\bf 0.01} \\
\hline
\end{tabular}
\label{tbl.Race}
\end{center}
\vspace{-0.1cm}
\end{table}

{\bf Performance with {\it Race} as protected attribute.} Table \ref{tbl.Race} presents the $Accuracy$, $\Delta_{DP}$, and $\Delta_{EO}$ scores of ten fair classification algorithms over two datasets of Adult and COMPAS.
Similar trends are observed for the classification performance comparison: \method\ achieves the highest $Accuracy$ values ($>$ 0.66), the smallest $\Delta_{DP}$ scores ($<$ 0.04), and the lowest $\Delta_{EO}$ values ($<$ 0.15) respectively, which are obviously better than all other methods. Especially, as shown the experiment on Adult, compared to the best competitors among ten group fairness algorithms, the $\Delta_{DP}$ and $\Delta_{EO}$ scores achieved by \method\ averagely achieves 10.0\% and 36.4\% improvement respectively. This demonstrates that the integration of the Gaussian parameter smoothing for providing a tight fairness guarantee and the parameter disparity term for helping recover the most accurate model for each group is able to make the classification results achieved by our input-agnostic certified group fairness method be accurate as well as fair.

\begin{figure}[t]
\mbox{
\subfigure[Accuracy]{\epsfig{figure=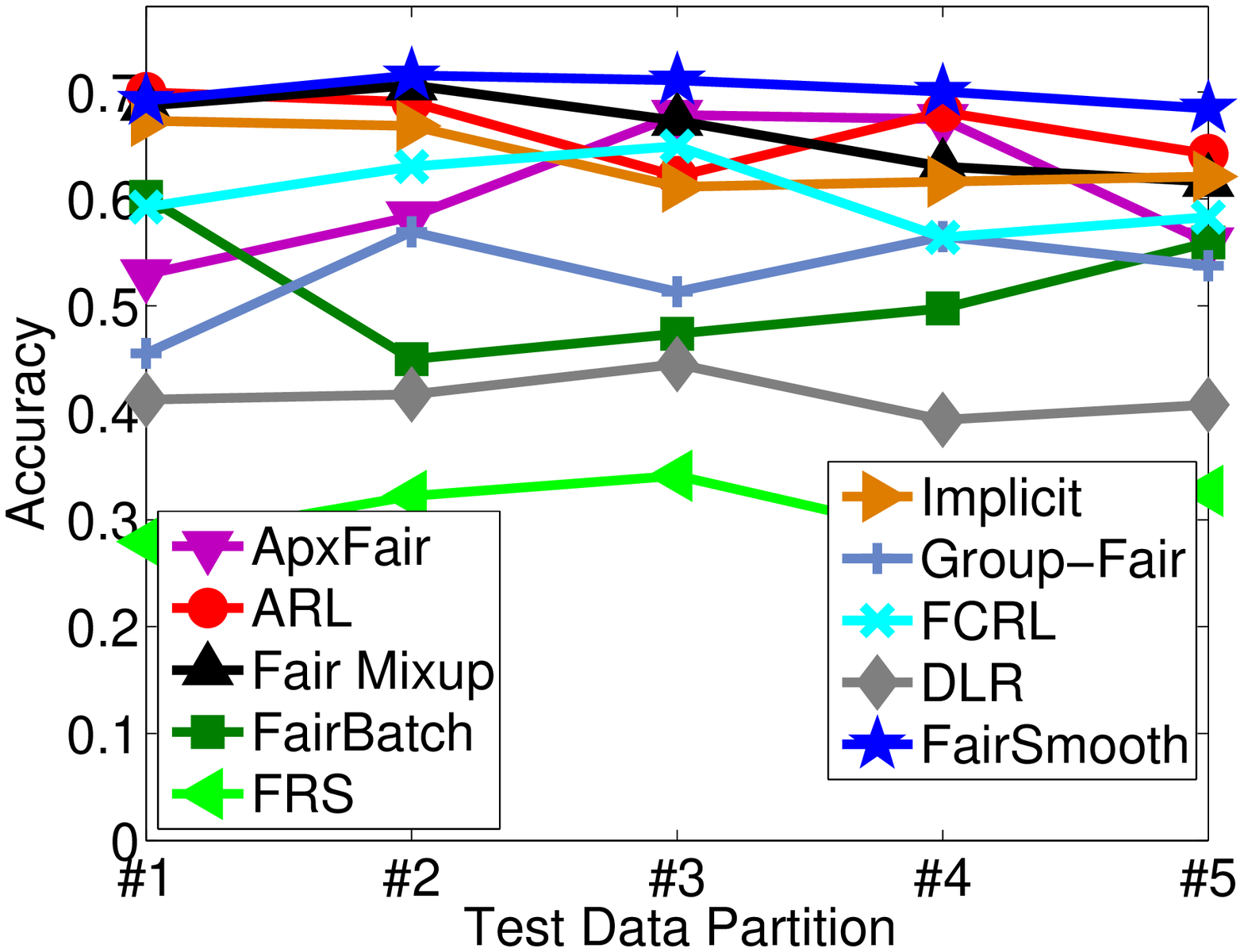, height=2in, width=0.33\linewidth}} \hspace{-0.225cm}
\subfigure[$\Delta_{DP}$]{\epsfig{figure=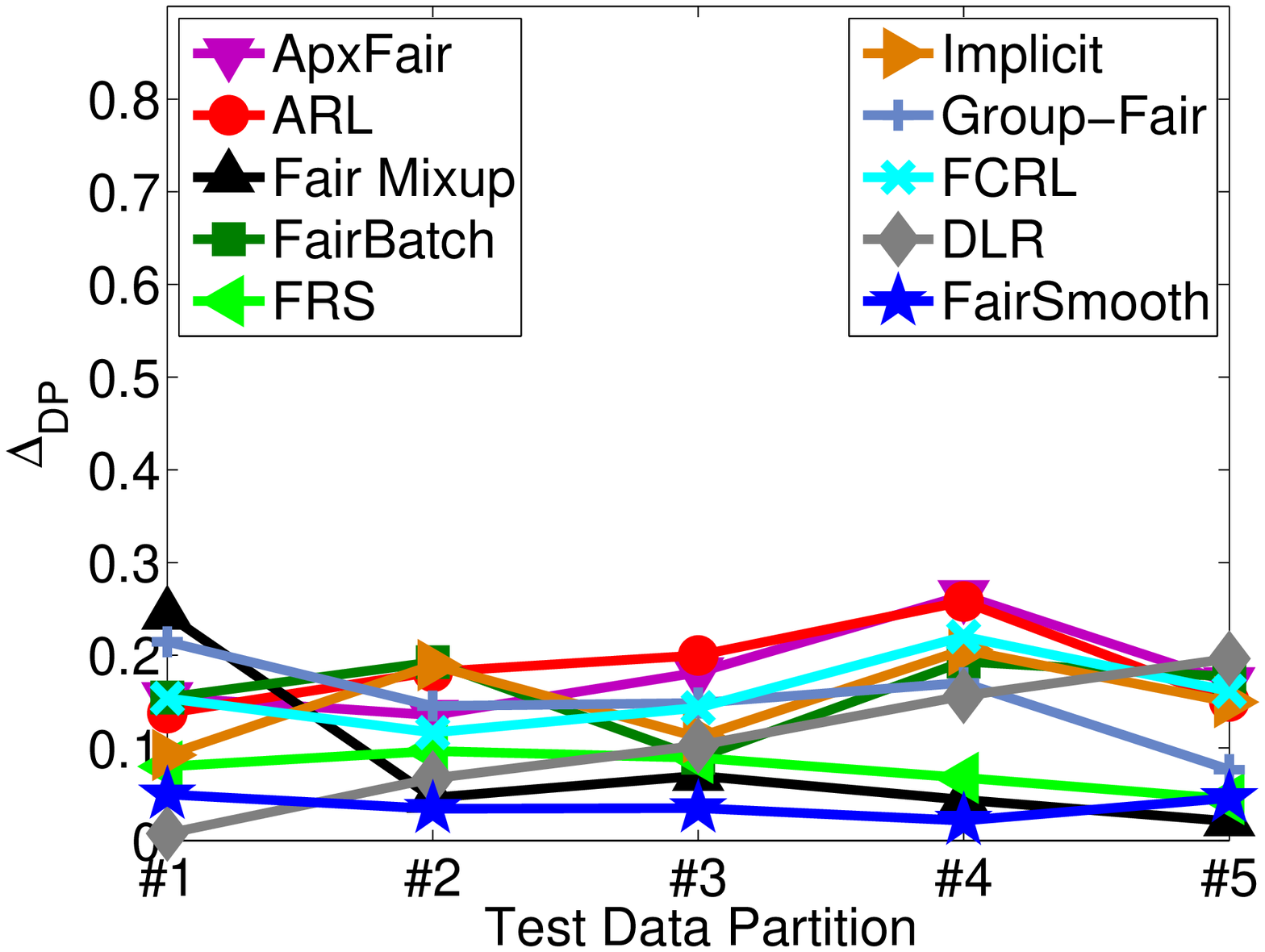, height=2in, width=0.33\linewidth}} \hspace{-0.225cm}
\subfigure[$\Delta_{EO}$]{\epsfig{figure=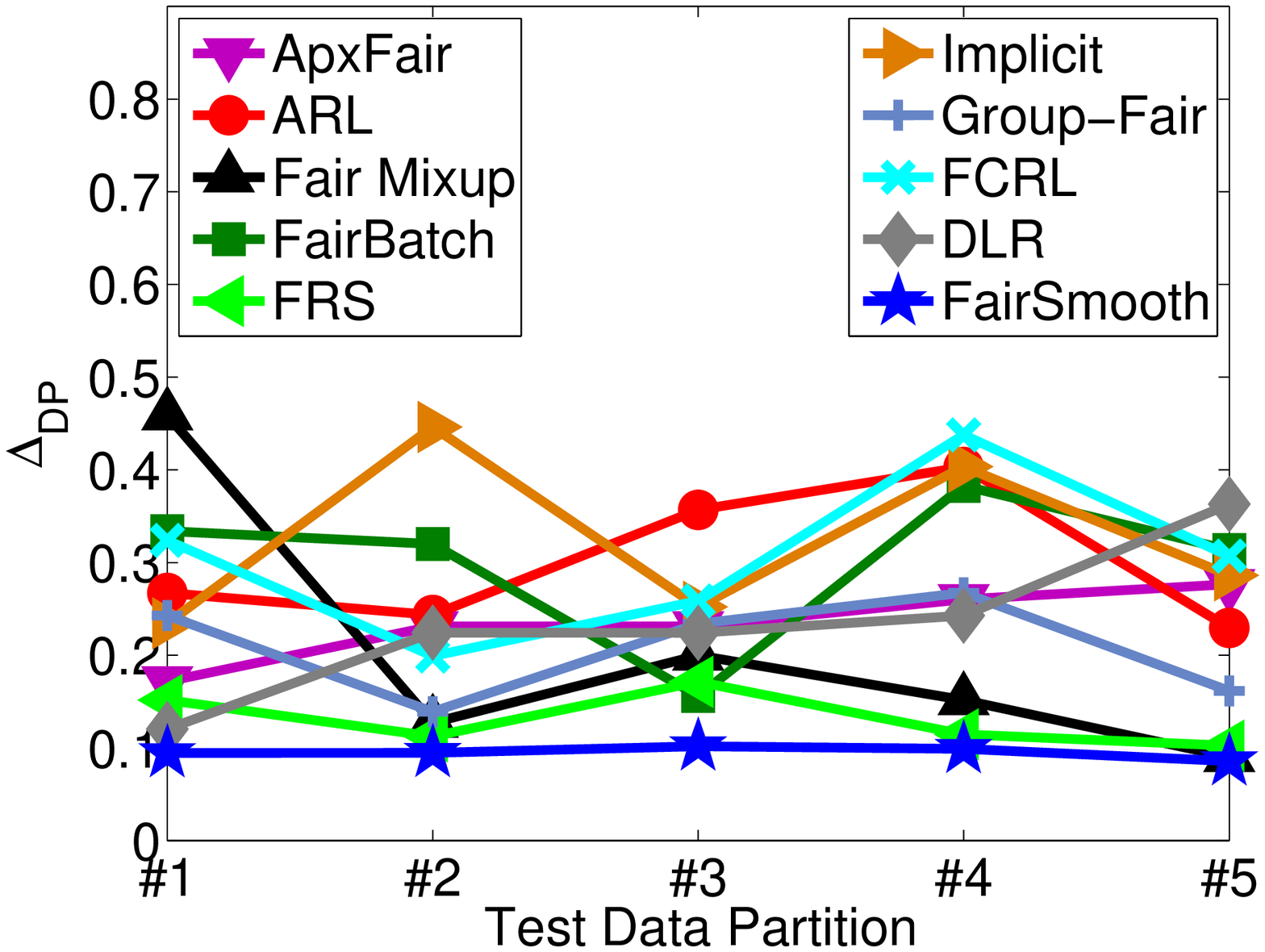, height=2in, width=0.33\linewidth}}}
\vspace{-0.2cm}
\caption{Performance with Varying Test Data Partitions on COMPAS}
\label{fig.DataPartition1}
\vspace{-0.3cm}
\end{figure}

{\bf Performance with varying test data partitions on COMPAS.} Figure~\ref{fig.DataPartition1} exhibits the fair classification performance with different test data on Adult. We have observed similar trends: the performance (both accuracy and fairness) achieved by our \method\ is insensitive to test data distributions. This is because the smooth classifiers are reformulated as output functions of a Nemytskii operator, which induces a Fr\'echet differentiable smooth manifold. The smooth manifold has a global Lipschitz constant that is independent of the domain of the input data, which derives the input-agnostic certified group fairness. In addition, \method\ still outperforms other fair classification methods in most experiments. The exceptional performance of \method\ over different datasets implies that our input-agnostic certified group fairness technique has great potential to achieve the superior certified fairness guarantees for any input data, which is desirable in practice.

\begin{figure}[t]
\mbox{
\hspace{1.4cm}
\subfigure[$\sigma$]{\epsfig{figure=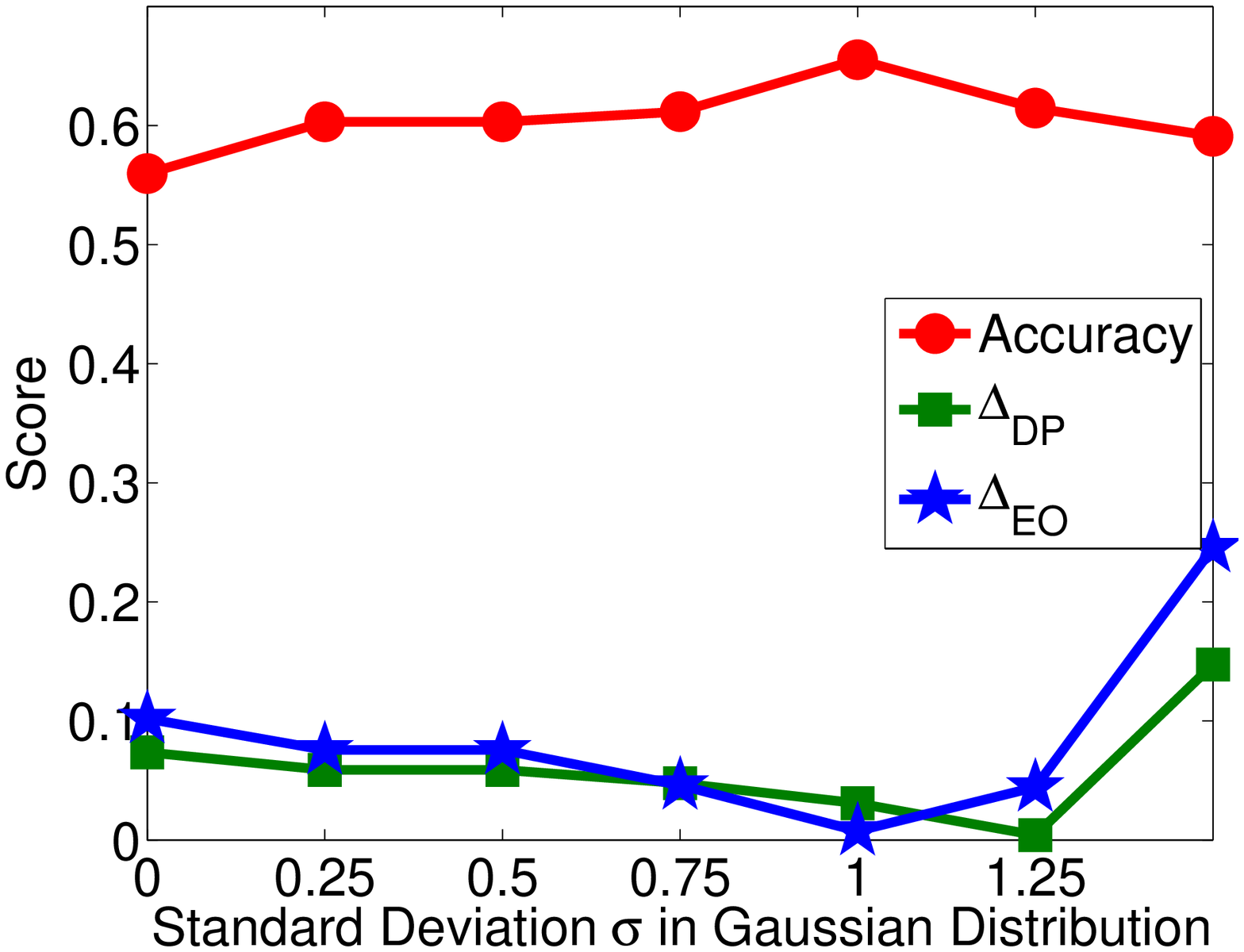, height=2in, width=0.395\linewidth}} 
\subfigure[$\alpha$]{\epsfig{figure=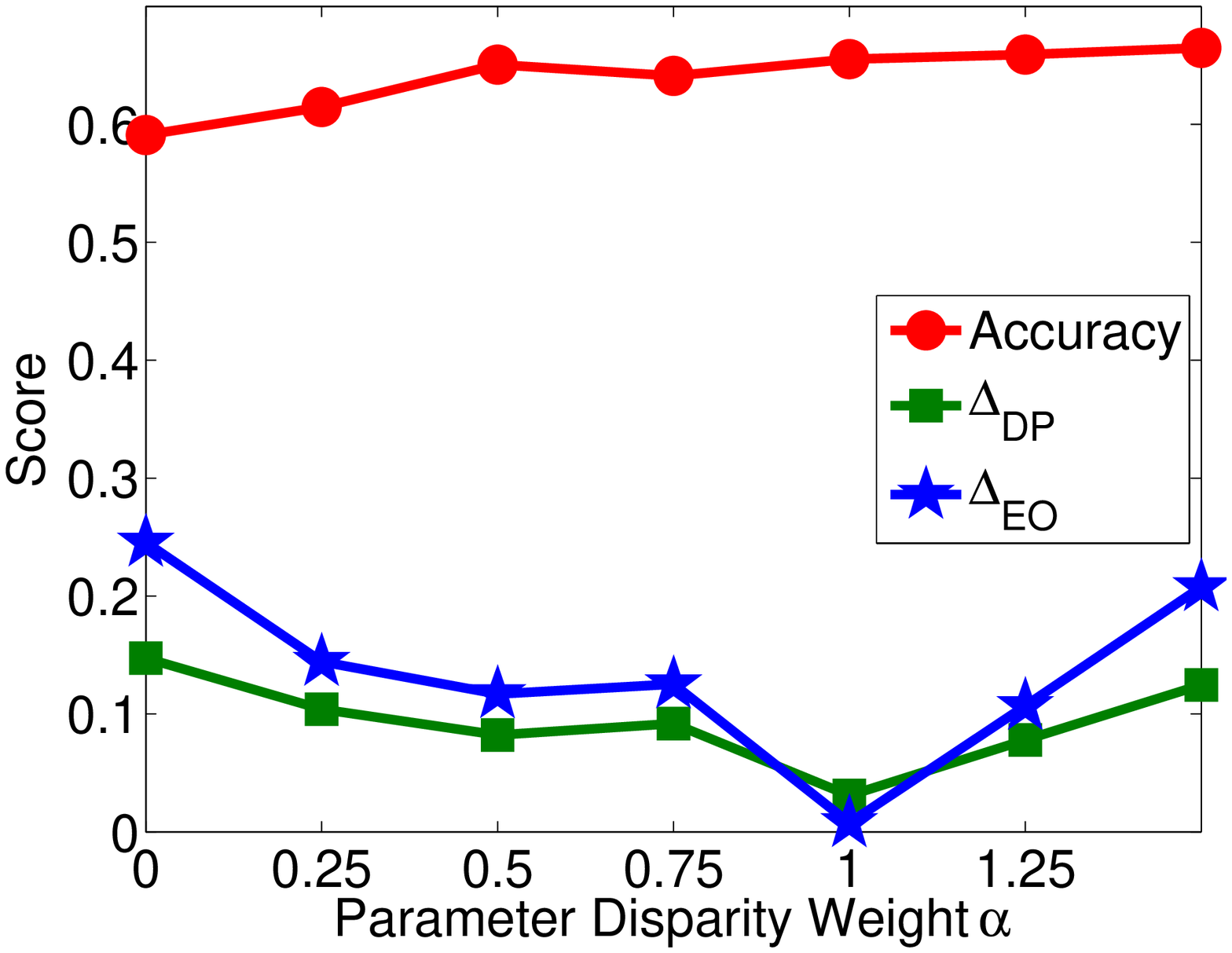, height=2in, width=0.395\linewidth}}}
\vspace{-0.2cm}
\caption{Performance with Varying Parameters on COMPAS}
\label{fig.Parameter1}
\vspace{-0.3cm}
\end{figure}

{\bf Impact of $\sigma$ and $\alpha$ on COMPAS.} Figure~\ref{fig.Parameter1} presents the influence of the standard deviation $\sigma$ in Gaussian distribution and the weight $\alpha$ of parameter disparity term in our \method\ model over the COMPAS dataset. We have witnessed the accuracy (or fairness) curves initially increase (or decrease) quickly and then become stable or even drop (or raise) when the parameters continuously increase. Initially, a large $\sigma$ can help leverage the strength of Gaussian parameter smoothing to produce a fair classifier. Later on, when $\sigma$ continues to increase and goes beyond some thresholds, too much Gaussian noise may ruin the prediction by the smooth classifier. In addition, $\alpha$ serves as a tradeoff hyperparameter to well balance the fairness and accuracy. Therefore, choosing a suitable parameter is important for the fair classifiers.

\vspace{-0.15cm}
\subsection{Experimental Details}\label{sec.ExperimentDetails}
\vspace{-0.1cm}

{\bf Environment.} Our experiments were conducted on a compute server running on Red Hat Enterprise Linux 7.2 with 2 CPUs of Intel Xeon E5-2650 v4 (at 2.66 GHz) and 8 GPUs of NVIDIA GeForce GTX 2080 Ti (with 11GB of GDDR6 on a 352-bit memory bus and memory bandwidth in the neighborhood of 620GB/s), 256GB of RAM, and 1TB of HDD. Overall, our experiments took about 2 days in a shared resource setting. We expect that a consumer-grade single-GPU machine (e.g., with a 1080 Ti GPU) could complete our full set of experiments in around tens of hours, if its full resources were dedicated.
The codes were implemented in Python 3.7.3 and PyTorch 1.0.14. We also employ Numpy 1.16.4 and Scipy 1.3.0 in the implementation. Since the datasets used are all public datasets and the hyperparameter settings are explicitly described, our experiments can be easily reproduced on top of a GPU server. We promise to release our open-source codes on GitHub and maintain a project website with detailed documentation for long-term access by other researchers and end-users after the paper is accepted.

{\bf Datasets.} The COMPAS dataset was introduced by ProPublica. It contains data from the US criminal justice system and was obtained by a public records request. The dataset contains personal information. To mitigate negative side effects, we delete the name, first, last and dob (date of birth) entries from the dataset before processing it further. We then exclude entries that do not fit the problem setting of predicting two year recidivism, following the steps of the original analysis~\cite{ProPublica}. Specifically, this means keeping only cases from Broward county, Florida, for which data has been entered within 30 days of the arrest. Traffic offenses and cases with insufficient information are also excluded. This steps leave 6,172 examples out of the original 7,214 cases.
The Adult dataset is available in the UCI data repository as well as multiple other online sources~\cite{UCI}. We use the dataset in unmodified form, except for binning some of the feature values.

{\bf Implementation.}
For six regular group fairness models of
ApxFair~\footnote{https://github.com/essdeee/Ensuring-Fairness-Beyond-the-Training-Data}, 
ARL~\footnote{https://github.com/google-research/google-research/tree/master/group\_agnostic\_fairness}, 
Fair Mixup~\footnote{https://github.com/chingyaoc/fair-mixup}, 
FairBatch~\footnote{https://github.com/yuji-roh/fairbatch},
fair-robust-selection (FRS)~\footnote{https://github.com/yuji-roh/fair-robust-selection}, and 
Implicit~\footnote{https://openreview.net/forum?id=pkh8bwJbUbL}, we used the open-source implementation and default parameter settings by the original authors for our experiments.
For three provable group fairness guarantee approaches of
Group-Fair~\footnote{https://github.com/vijaykeswani/FairClassification}, 
FCRL~\footnote{https://github.com/umgupta/fairness-via-contrastive-estimation}, and 
DLR~\footnote{https://github.com/vijaykeswani/Noisy-Fair-Classification}, 
we also utilized the same model architecture as the official implementation provided by the original authors and used the same datasets to validate the fairness of these fair classification models in all experiments. All hyperparameters are standard values from reference codes or prior works. The above open-source codes from the GitHub are licensed under the MIT License, which only requires preservation of copyright and license notices and includes the permissions of commercial use, modification, distribution, and private use.

For our proposed input-agnostic certified group fairness algorithm, we performed hyperparameter selection by performing a parameter sweep on sample numbers $N \in \{500, 1,000, 5,000, 10,000, 50,000, 100,000\}$ in the integral calculation, parameter disparity weight $\alpha \in \{0.25, 0.5, 0.75, 1, 1.25, 1.5\}$ in the \methodI\ model, training epochs of the smooth classifiers $\in \{10, 20, 40, 80, 160, 320\}$, batch size for training the smooth classifiers $\in \{16, 32, 48, 64, 128, 256\}$, and learning rate $\in \{0.0005, 0.001, 0.005, 0.01, 0.05, 0.1\}$. In this paper we smooth the base classifier with Gaussian parameter augmentation (i.e., perturbations) at variance $\sigma^2$. We select $\sigma$ from a range $\sigma \in \{0.25, 0.5, 1, 1.25, 1.5, 1.75, 2\}$. We select the best parameters over 50 epochs of training and evaluate the model at test time. 
After the hyperparameter selection, the model was trained for 320 iterations, with a batch size of 128, and a learning rate of 0.05.